\documentclass{article}

\usepackage{hyperref}

\usepackage{amsmath,amsfonts,bm}

\def\eqref#1{equation~\ref{#1}}

\def\1{\bm{1}}

\def\vzero{{\bm{0}}}

\def\vbeta{{\bm{\beta}}}

\def\ve{{\bm{e}}}

\def\vk{{\bm{k}}}

\def\vm{{\bm{m}}}

\def\vq{{\bm{q}}}

\def\vu{{\bm{u}}}
\def\vv{{\bm{v}}}
\def\vw{{\bm{w}}}
\def\vx{{\bm{x}}}

\def\vz{{\bm{z}}}

\def\mH{{\bm{H}}}
\def\mI{{\bm{I}}}

\def\mK{{\bm{K}}}

\def\mP{{\bm{P}}}
\def\mQ{{\bm{Q}}}

\def\mU{{\bm{U}}}

\def\mW{{\bm{W}}}
\def\mX{{\bm{X}}}

\def\mSigma{{\bm{\Sigma}}}

\DeclareMathAlphabet{\mathsfit}{\encodingdefault}{\sfdefault}{m}{sl}
\SetMathAlphabet{\mathsfit}{bold}{\encodingdefault}{\sfdefault}{bx}{n}

\def\gL{{\mathcal{L}}}
\def\gM{{\mathcal{M}}}

\def\gS{{\mathcal{S}}}

\newcommand{\E}{\mathbb{E}}

\newcommand{\R}{\mathbb{R}}

\DeclareMathOperator{\Tr}{Tr}

 \usepackage[preprint]{conf/icml2026}

\usepackage{microtype}
\usepackage{graphicx}
\usepackage{subcaption}
\usepackage{booktabs}

\usepackage{amsmath}
\usepackage{amssymb}
\usepackage{mathtools}
\usepackage{amsthm}
\usepackage{thmtools} 
\usepackage[capitalize,noabbrev]{cleveref}
\usepackage{algorithm}
\usepackage{algorithmic}
\usepackage{enumitem}

\usepackage{xspace}
\usepackage{multirow}
\usepackage{appendix}

\usepackage{xfp} 
\usepackage{tcolorbox}
\tcbuselibrary{listings, breakable} 

\newif\ifshowcomments  

\theoremstyle{plain}
\newtheorem{theorem}{Theorem}[section]

\newtheorem{lemma}[theorem]{Lemma}

\theoremstyle{definition}
\newtheorem{definition}[theorem]{Definition}

\newtheorem{assumption}[theorem]{Assumption}
\theoremstyle{remark}

\DeclareMathOperator{\ent}{\mathbf{Entropy}}

\makeatletter
\patchcmd{\l@section}
  {\addvspace{1.0em \@plus\p@}}
  {}
  {}{}
\makeatother

\newcommand{\percentimpr}[2]{\fpeval{round(100*(#2-#1)/#1)}}
\newcommand{\percentdecr}[2]{\fpeval{round(100*(-#2+#1)/#1)}}

\showcommentstrue  

\newcommand{\papertitle}{Data Distribution as a Lever for Guiding Optimizers Toward Superior Generalization in LLMs}
\icmltitlerunning{\papertitle}

\begin{document}

\twocolumn[
  \icmltitle{\papertitle}

    \icmlsetsymbol{equal}{*}
    
    \begin{icmlauthorlist}
        \icmlauthor{Tushaar~Gangavarapu}{equal,utaustin}
        \icmlauthor{Jiping~Li}{equal,ucla}
        \icmlauthor{Christopher~Vattheuer}{equal,ucla}
        \icmlauthor{Zhangyang~Wang}{utaustin}
        \icmlauthor{Baharan~Mirzasoleiman}{ucla}
    \end{icmlauthorlist}
    
    \icmlaffiliation{ucla}{University of California, Los Angeles}
    \icmlaffiliation{utaustin}{University of Texas at Austin}

    \icmlcorrespondingauthor{Jiping~Li}{jipingli0324@g.ucla.edu}
    \icmlkeywords{Sharpness-Aware Minimization, Generalization, Simplicity bias, Data-centric ML}

    \vskip 0.3in
]

\printAffiliationsAndNotice{\icmlEqualContribution}


\addtocontents{toc}{\protect\setcounter{tocdepth}{-10}}

\begin{abstract}
    Can modifying the training data distribution guide optimizers toward solutions with improved generalization when training large language models (LLMs)? In this work, we theoretically analyze an in-context linear regression model with multi-head linear self-attention, and compare the training dynamics of two gradient based optimizers, namely gradient descent (GD) and sharpness-aware minimization (SAM), the latter exhibiting superior generalization properties but is prohibitively expensive for training even medium-sized LLMs. 
We show, for the first time, that SAM induces a lower simplicity bias (SB)—the tendency of an optimizer to preferentially learn simpler features earlier in training—and identify this reduction as a key factor underlying its improved generalization performance. Motivated by this insight, we demonstrate that altering the training data distribution by upsampling or augmenting examples learned later in training similarly reduces SB and leads to improved generalization. 
Our extensive experiments show that our strategy improves the performance of multiple LLMs—including Phi2-2.7B , Llama3.2-1B, Gemma3-1B-PT, and Qwen3-0.6B-Base—achieving relative accuracy gains up to $18\%$ when fine-tuned with AdamW and Muon on mathematical reasoning tasks.

\end{abstract}


\section{Introduction}
\label{sec:intro}

Training machine learning models with different optimizers yields solutions with different generalization properties. 
Recently, 
f-SAM, a functional variant of Sharpness-Aware Minimization (SAM), is shown to significantly enhance the generalization capabilities of large language models (LLMs) 
\citep{singh2025avoiding}. 
However, f-SAM is very expensive requiring over 6x time and 2x memory compared to AdamW, thus becomes prohibitively expensive for training even medium sized LLMs.
This raises the following questions: \textit{what derives SAM to find more generalizable solutions? Can we change the training data distribution to make training dynamics of cheaper gradient-based methods, such as Adam-W \citep{loshchilov2017decoupled} and Muon \cite{muon}, more similar to SAM and improve the generalization performance? }

Answering this question requires theoretically analyzing training dynamics of SAM for transformers. 
In convolutional models, SAM’s behavior can often be interpreted through the lens of the Hessian spectrum \citep{wen2022does,bartlett2023dynamics}, gradient noise scale \citep{chen2023does}, and local geometry of the loss landscape \cite{andriushchenko2022towards}, where improvements are linked to flatter minima and enhanced robustness. However, in transformers, the autoregressive training objective, non-local dependencies induced by self-attention, and the highly anisotropic gradient structure create a fundamentally different optimization regime. 
Consequently, 
extending these insights to language models requires new analytical tools that account for the unique training dynamics of transformers.

In this work, we study the gradient descent dynamics of multi-head linear self-attention trained for in-context linear regression, which is a popular framework for studying transformers \citep{zhang2025training,huangtransformers,javanmard2025understandingroletrainingdata}.
In this setting, 
the loss exhibits 
multiple abrupt drops 
each corresponding to learning a feature (a direction of the data covariance matrix) in the data \citep{zhang2025training}. 
We show that 
SAM induces drops in the training loss at a more uniform speed. That is, SAM slows down learning the simplest features and speeds up learning more complex ones, indicating its reduced Simplicity Bias (SB)   \citep{kalimeris2019sgd,hu2020surprising}. 
To the best of our knowledge, we are the first to reveal lower simplicity bias of SAM on language-based models.

Motivated by our theoretical analysis, we investigate ways to reduce the simplicity bias of training Large Language Models (LLMs) when training with cheaper gradient methods, such as Adam-W \cite{loshchilov2017decoupled} and Muon \cite{muon}. To achieve this, we aim to speed up learning more complex features, by identifying and upsampling or augmenting examples that only contain complex features. Such examples are learned later during training. 
To find such examples, we train a smaller proxy LLM for a small number of iterations and 
track the loss, forming a loss trajectory for all examples. 
Examples with simple features that are learned quickly have a reducing loss trajectory early in training.
To reduce SB, we cluster the loss trajectories (or their final values) into two groups and upsample or augment the group of examples with the higher average loss. Our experiments show that training on the upsampled data significantly boosts the performance of various LLMs, including Qwen3-0.6B-Base \citep{yang2025qwen3}, Llama3.2-1B \citep{dubey2024llama}, Gemma3-1B-PT \citep{team2025gemma}, and Phi2-2.7B \citep{javaheripi2023phi} fine-tuned on MathInstruct \citep{yue2023mammoth} with AdamW and Muon optimizers, achieving relative accuracy gains up to $18\%$.
\section{Related work} 

\textbf{Theory of Transformers.}
There has been a surge of theoretical studies investigating various properties of transformers in the in-context learning paradigm. This includes convergence \citep{zhang2024context,zhang2024trained,ren2024learning,fu2024transformers,zhang2025training}, generalization \citep{wu2023many,mahankali2023one,duraisamy2024finite,lu2025asymptotic,abedsoltan2024context,frei2024trained}, expressivity \citep{vladymyrov2024linear,gatmiry2024role}, 
and loss landscape properties \citep{mahankali2023one, li2024fine}. 
Most relevant to us is the recent study of \cite{zhang2025training} which studied transformers with multi-head attention. 
By reducing the high-dimensional training dynamics to scalar ordinary differential equations through an ansatz, they found 
exponentially many fixed points in the training dynamics, each corresponding to learning a direction in the data covariance matrix. 

In our work, we analyze the difference in training dynamics of SAM and GD 
and show that 
SAM learns difference directions of the data covariance matrix more homogeneously during the training, indicating its lower simplicity bias. Our work is the first to reveal lower simplicity bias of SAM as a key factor contributing to its superior generalization performance on transformer models.

\textbf{Simplicity Bias.} 
Gradient methods has an inductive bias towards learning simpler solutions. Simplicity bias has long been conjectured to be the reason for the superior generalization performance of overparameterized fully-connected and convolutional models by providing implicit regularization \citep{belkin2019reconciling,gunasekar2017implicit,hu2020surprising,nakkiran2021deep}. 
However, if gradient-based optimizers have a similar inductive bias for 
language models trained with self-supervised techniques remained an open question.
The very recent empirical study of \citep{rende2024distributional} provided the first empirical evidence by demonstrating that transformers sequentially learn 
low-degree interactions followed by high-degree interactions. However, this phenomenon has remained theoretically unexplored.

In our work, we provide the first theoretical study confirming that SAM has a lower simplicity bias compared to GD when training transformers, and identify this as a key contributor to its superior generalization performance.
\section{Preliminaries}
\label{sec:prelim}


%
%




To understand the evolving dynamics of gradient-based training, we consider the widely adopted setting of a multi-head linear self-attention model for the in-context linear regression task \citep{garg2022can,akyurek2022learning,von2023transformers,ahn2023transformers,bai2023transformers,zhang2025training}.
Although linear self-attention, i.e., self-attention \textit{without} the softmax, may seem oversimplified, \citet{loshchilov2017decoupled,ahn2023transformers} show that it is particularly well suited for studying the optimization properties of transformers.

\textbf{Data Model.} 
Let $\vx_i\in\R^d$ denote the covariates drawn from $\mathcal{N}(\bm{0},\mSigma),$ and $\vw_\star\in\R^d$ from $\mathcal{N}(\bm{0}, \mI_d).$ The ground-truth responses are computed as $y_i=\vw_\star^\top\vx_i \in \R$. Now we model in-context next-token prediction with input sequence $\{\vx_1, y_1, \vx_2, y_2, \dots, \vx_{N}, y_N, \vx_q\}$, where $\vx_q$ is the \textit{query} token whose response is to be estimated. We define: 
\begin{equation}\label{eq:data}
    \mX = \begin{bmatrix}
     \vx_1&  \vx_2& \dots&  \vx_N & \vx_q \\
      y_1& y_2&  \dots&  y_N & 0 
    \end{bmatrix} \in \mathbb{R}^{(d+1)\times(N + 1)},
\end{equation}
where the bottom-right entry is set to zero to indicate the unknown response $y_q$ associated with $\vx_q$. A model $f:\mathbb{R}^{(d+1) \times (N + 1)} \to \mathbb{R}^{(d+1) \times (N + 1)}$ predicts $y_q$ by extracting $\hat{y}_q := f(\mX)_{d+1, N+1}$ i.e., the bottom right entry below $\vx_q$. 


\textbf{Multi-Head Linear Self-attention.} We consider the linear self-attention model with $H$ heads and a residual connection.
For $\mX \in \mathbb{R}^{(d+1) \times (N + 1)},$ we define: 
\[
    \texttt{ATTN}_S(\mX)   = \mX + \sum_{i=1}^H\frac{1}{N}\mW_i^V\mX\mX^\top{\mW_i^{K}}^\top\mW_i^{Q}\mX, 
\]
where $\mW_i^V\in \R^{(d+1) \times (d+1)},\mW_i^K,\mW_i^Q \in \R^{r \times (d+1)}$ are the trainable value, key, and query matrices for head $i.$ In practice, the model is trained with a dataset of $P$ $\{$token sequence, query$\}$ pairs $\{\mX_i, y_{q,i}\}_{i=1}^P$, where each data point can have independently sampled weight vector $\vw_{\star, i} \sim \mathcal{N}(\bm{0}, \mI_d)$.

\begin{figure*}[t]
    \centering
    \includegraphics[width=0.9\textwidth]{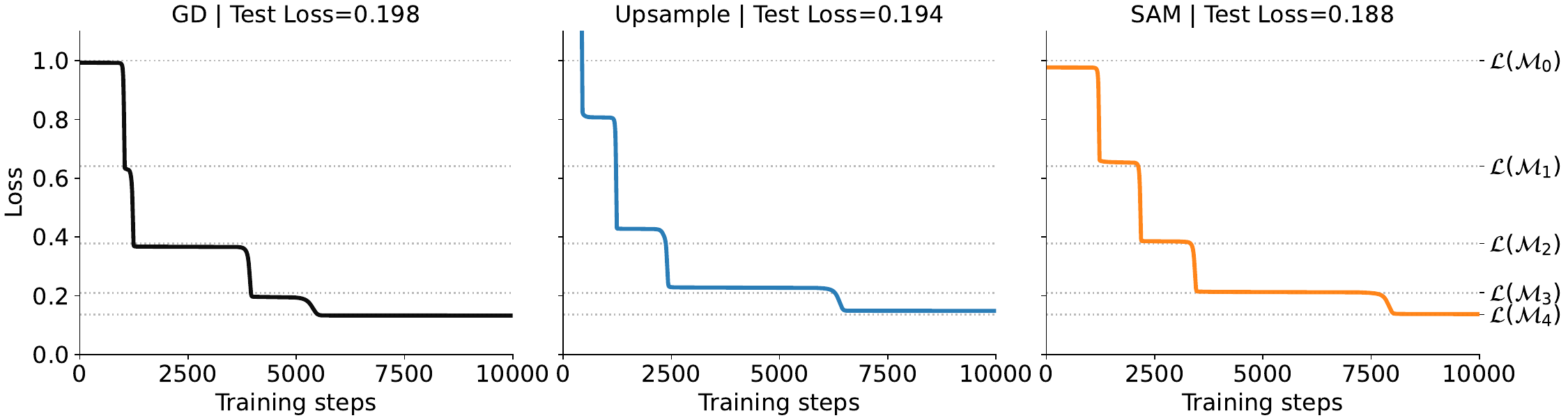}
    \caption{Comparison of training loss for (left) GD, (right) SAM, and (middle) GD + upsampling (our method). We perform 25 independent runs and report the median test loss across all 25 runs on top of each plot. Both SAM and GD+upsampling exhibit  lower simplicity bias, i.e., more uniform feature learning, (demonstrated by abrupt drops in the loss) than GD, and obtain lower test loss. }
    
    \label{fig:toy_results}
\end{figure*}
For the remainder of our analysis, we will assume rank-one ($r = 1$) key and query matrices $\mW_i^K,\mW_i^Q$ for each head; this follows from \citet{zhang2025training}, which notes that the rank-one case generalizes well to higher ranks and offers relevant insights.


\begin{lemma}[Equivalence to three-layer convolutional network (CNN); from $\S$4.1 of \citet{zhang2025training}]
\label{lem:equivalence_to_CNN}
    Under the assumption of rank-one key and query matrices and starting from a suitable initialization,\footnote{
        $\mW_i^V=\begin{bmatrix} 
            * & *\\ 
            \vzero & v_i
        \end{bmatrix}, 
        \mW_i^K=\begin{bmatrix} \vk_i^\top & 0 \end{bmatrix}, 
        \mW_i^Q=\begin{bmatrix} \vq_i^\top & * \end{bmatrix}$, where * indicates values that do not matter;
        the weights initialized to zero are not required to achieve global minimum loss on the in-context linear regression task \citep{ahn2023transformers,zhang2024trained,zhang2025training}.
    } 
     the prediction from $\texttt{ATTN}_S(\mX)$ admits the following decomposition into a three-layer CNN: 
    \[
        \texttt{ATTN}_S(\mX)_{d+1, N+1} = \sum_{i = 1}^Hv_i\vq_i^\top\mK_i\vz, 
    \]
    \[
        \mK_i = \begin{bmatrix}
            \vk_i^\top  & \vzero_d^\top  & \dots & \vzero_d^\top \\
            \vzero_d^\top & \vk_i^\top & \dots & \vzero_d^\top \\ 
            \vdots & \vdots & \ddots & \vdots \\ 
            \vzero_d^\top & \vzero_d^\top & \dots & \vk_i^\top \\ 
        \end{bmatrix} \in \mathbb{R}^{d \times d^2},
    \]
where $\vz \in \mathbb{R}^{d^2}$ is some cubic feature map of input $\mX$, and $v_i\in\R$, $\vq_i\in\R^d$, and $\vk_i\in\R^d$ contain weights from $\mW_h^K$, $\mW_h^Q$, and $\mW_h^V$ respectively.
    
    Consequently, the summation over $H$ heads of $v_i\vq_i^\top\mK_i\vz$ can be interpreted as applying $H$ convolutional filters, with the key, query, and value weights forming the three layers.


\end{lemma}

This three-layer CNN model is arguably the simplest nonlinear proxy for our analysis that is both sufficiently complex and tractable. More discussion of \cref{lem:equivalence_to_CNN} in \citep{zhang2025training} is included in Appendix \ref{subsec:proof_of_equivalence}.



\textbf{Gradient Flow Dynamics of Early Training with GD.} 
The linear self-attention model with learnable parameters $\vw$ is trained to minimize the mean-squared loss $\gL(\vw^{(t)}) = \E[(y_q-\hat{y}_q)^2]$ using gradient descent (GD) at iteration $t$ with learning rate $\eta$: 
\begin{align}
\label{eq:GD_updates}
    \vw^{(t+1)} &= \vw^{(t)} - \eta\nabla_{\vw^{(t)}}\gL(\vw^{(t)}). 
\end{align}


To reason about the training dynamics, \citet{zhang2025training} established the GD dynamics on $\texttt{ATTN}_S$ via gradient flow, which can be seen as the continuous-time limit of gradient descent as $\eta\to0$\footnote{We use both matrix $\mW$ and vector $\vw$ to denote weights interchangeably, whichever is more appropriate in the context.}: 
\begin{align*}
    \tau\dot{\mW}_{\text{GD}} = \tau\frac{\mathrm{d}\mW_{\text{GD}}}{\mathrm{d}t} = -\frac{1}{2}\frac{\partial \gL}{\partial\mW_{\text{GD}}}= \E\left[(y_q-\hat{y}_q)\frac{\partial \hat{y}_q}{\partial\mW_{\text{GD}}}\right],
\end{align*}
for some time constant $\tau$.

\textbf{Feature Learning and Simplicity Bias in GD.} 
During training,
$\texttt{ATTN}_S$ activates one head at a time to learn the most prominent 
feature among the ones that are not learned yet. Formally, for data covariance $\mSigma = \sum_{i=1}^d \lambda_i \ve_i \ve_i^\top$, where each eigenvector $\ve_i$ captures a feature with strength $\lambda_i$, a head $i$ becomes active and learns the current eigenvector $\ve_h$\footnote{%
    For notational convenience, the attention heads are permuted so that the head $i$ learns the $i$-th eigenvector $\ve_i$
}  that has the largest eigenvalue not yet learned. Its key and query weights grow the fastest along the eigenvector\footnote{Specifically, if $v_h > 0$, both $\vk_h$, $\vq_h$ approximately grow along $\ve_h$; if $v_h < 0$, $\vk_h$ and $\vq_h$ grow in opposite directions, one along $\ve_h$ and the other along $-\ve_h$. The upcoming ansatz implicitly assumes the former.}. 

As shown in \cref{fig:toy_results} (left), this results in a stage-wise training loss, with multiple abrupt drops separated by plateaus.
Every drops indicates learning a new feature by $\texttt{ATTN}_S$ via saddle-to-saddle movement and during plateaus the model lingers near a fixed point. 

 As progressive learning in $\texttt{ATTN}_S$ always targets the most prominent eigenvector (easiest feature) among remaining ones, this implies the existence of Simplicity Bias in GD.\looseness=-1

\textbf{Global Dynamics $\iff$ Last-Layer $v$ Dynamics.} 
During the early plateau of learning $\ve_{m+1}$, the layer weights approximately evolve as ~\citep{zhang2025training}:
 \begin{equation} \label{eq:ansatz1}
     \vk_i = \vq_i = v_i\ve_i = \left(\lambda_i + \frac{\lambda_i + \Tr(\mSigma)}{N}\right)^{-1/3} \ve_i, 
 \end{equation}
 \[
     \text{for} \ 1 \leq i  \leq m \text{ if $m \neq 0$  (already learned features)} ; 
 \]
 \begin{equation} \label{eq:ansatz2}     
    \vk_{m+1} = \vq_{m+1} = v_{m+1}(t)\ve_{m+1}, \text{for current feature}; 
 \end{equation}
 \begin{equation} \label{eq:ansatz3}
     \vk_i = \vq_i =  \bm{0},  v_i = 0, \ \text{for}\ m+2 \leq i \leq H.    
 \end{equation}


Eq.~\ref{eq:ansatz2} shows that only the current active head has dynamics ($v_{m+1}(t)$) and the other heads remain approximately unchanged (no dependence on $t$). 
We see that the key and query weights only depend on the value $v_i$ and the current feature $\ve_i,$ which determine the norm and direction respectively. This results in a key simplification, as we can capture the global training dynamics only by presenting the $v$ (last-layer) dynamics. 
While Eq. (\ref{eq:ansatz1})-(\ref{eq:ansatz3}) only captures the early learning phase of each feature, \cref{fig:toy_results} confirms a similar behavior throughout the training.
\section{Analyzing Feature Learning for SAM}
In this section, we first introduce Sharpness-Aware Minimization (SAM) and then analyze its early feature learning dynamics. Finally, we build on our analysis to compare simplicity bias of SAM with GD.

\subsection{Sharpness-Aware Minimization (SAM)}

While GD can effectively find weights that yield low loss, such local minima may have noisy local landscapes 
with suboptimal generalization performance. 
To mitigate this issue, SAM seeks more robust parameters that reside in ``flatter" neighborhoods and thus ensures superior generalization performance. \looseness=-1
Formally, we have the following weight update rules for loss $\gL(\vw^{(t)})$ at iteration $t$ with learning rate $\eta$ and perturbation strength $\rho$: 
\begin{align} \label{eq:GD_and_SAM_updates}
    &\textbf{SAM: } \  \vw^{(t + 1)} = \vw^{(t)} - \eta \nabla_{\vw^{(t)}}\gL(\vw^{(t)} + \bm{\varepsilon}^{(t)}), \quad  \notag \\
    &\text{where } \ \bm{\varepsilon}^{(t)} = \rho\frac{\nabla_{\vw^{(t)}}\gL(\vw^{(t)})}{\|\nabla_{\vw^{(t)}}\gL(\vw^{(t)})\|_2} \ \ \text{for $\rho > 0$}.
\end{align}
Note that the normalized gradient is detached from the gradient computation.

Technically, SAM performs gradient descent on weights perturbed in the gradient ascent direction, i.e. the most adversarial local perturbation. This mechanism automatically implements a firewall against sharp minima by discouraging convergence to regions where the loss landscape could change steeply.




\textbf{Gradient Flow for Early Training with SAM. }
%
Given a small perturbation radius $\rho$, we can approximate the SAM update (\ref{eq:GD_and_SAM_updates}) via a first-order Taylor expansion: 
\allowdisplaybreaks
\begin{align} \label{eq:SAM_first_order}
    \nabla_{\vw^{(t)}}\gL&(\vw^{(t)} + \bm{\varepsilon}^{(t)}) \notag \\
    =\  &\nabla_{\vw^{(t)}}\gL(\vw^{(t)}) + \mH^{(t)}\bm{\varepsilon}^{(t)} + O\left(\rho^2\right) \notag \\ 
    \approx \ & \left(\mI + \frac{\rho}{\|\nabla_{\vw^{(t)}}\gL(\vw^{(t)})\|_2}\mH^{(t)}\right)\nabla_{\vw^{(t)}}\gL(\vw^{(t)}) \notag \\ 
     =\ & (\mI + \mP)\nabla_{\vw^{(t)}}\gL(\vw^{(t)}), 
\end{align}
where $\mP$ is the ``\textit{perturbation matrix}" and $\mH^{(t)} = \mH({\vw^{(t)}}) :=\nabla_{\vw^{(t)}}^2\gL(\vw^{(t)}) $ denotes the Hessian matrix, from which the SAM optimizer ``borrows curvature information" to help escape from sharp areas.

We now consider $\texttt{ATTN}_S$ trained with SAM via this approximation (\ref{eq:SAM_first_order}), which holds for small perturbation $\rho$. We study the gradient flow of loss $\gL =\mathbb{E}[(y_q - \hat{y}_q)^2]$ under SAM. In particular, we are interested in: 
\begin{equation} \label{eq:SAM_GD_dynamics}
    \tau\dot{\mW}_{\text{SAM}} = -\frac{1}{2}\frac{\partial \gL}{\partial \mW}\left(\mW + \frac{\rho \nabla_{\mW}\gL}{\|\nabla_{\mW}\gL\|}\right)\approx (\mI + \mP)\tau\dot{\mW}_{GD}, 
\end{equation}
where $\tau$ is the time constant. The last approximation follows from the first-order expansion (\ref{eq:SAM_first_order}) and the perturbation $\mP$ detached from gradient. 

\subsection{Early Training ODE Dynamics for SAM}\label{sec:methods}
Next, we theoretically analyze feature learning for the multi-head self-attention model trained with SAM.

First, we preset the following lemma which ensures that the model under SAM maintains a similar landscape to GD in terms of the distribution of fixed points. 

\begin{lemma}[Manifolds of saddle points] \label{lem:static_saddle_points} Let $\mP$ be the perturbation matrix induced by SAM in the first-order approximation, and $\sigma_{\text{min}}(\mI - \mP) > 0$ during the gradient flow. Then GD and SAM dynamics ($\dot{\mW}_{\text{SAM}}$ and $\dot{\mW}_{\text{GD}}$) share the same space of fixed point manifolds (see Appendix \ref{subsec:manifolds}) for $\texttt{ATTN}_S(\cdot)$, where $\mW = \{v_i, \vq_i, \vk_i : i = 1, \dots, H\}$. 
\end{lemma}

Hence, under this non-degeneracy condition, the SAM flow does not introduce new saddle points or remove existing ones: any 
difference in feature learning for SAM vs GD 
as we show next, can be attributed to the optimizer-induced dynamics, rather than to a change in the ``geometry of equilibria".

Next, by explicitly computing the Hessian matrices and simplifying, we derive early training dynamics as an ODE when each head starts learning. Similar to ansatz (\ref{eq:ansatz1})-(\ref{eq:ansatz3}), we only present the last-layer $v$ dynamics, which can capture global behaviors: 

\begin{theorem}[Early Training Dynamics of SAM] \label{thm:SAM_ODE}
    Under the setting of Lemma \ref{lem:static_saddle_points}, ansatz (\ref{eq:ansatz1})-(\ref{eq:ansatz3}), 
    and with data covariance matrix $\mSigma = \sum_{i=1}^d \lambda_i \ve_i \ve_i^\top$ indexed by decreasing \textit{distinct bounded} eigenvalues, the early training dynamics of GD and first-order SAM expansion (\ref{eq:SAM_first_order}) on $\texttt{ATTN}_S(\mX)$ when learning the $i$-th feature/eigenvector ($i = 1, \dots, \min\{H, d\}$) can be approximately captured by the one scalar-variable ODEs: 
    \begin{equation} \label{eq:SAM_ODE}
        \textbf{SAM} \text{ (ours)}: \tau\dot{v}_i = \lambda_i^{2} v_i^{2}  - \frac{2\rho}{\sqrt{3}}\lambda_i^2v_i
    \end{equation}
    \begin{equation} \label{eq:GD_ODE}
        \textbf{GD} \text{~\citep{zhang2024context}}: \tau\dot{v}_i = \lambda_i^{2} v_i^{2}
    \end{equation}
    under sufficiently small weight initializations. 
\end{theorem}

The proof can be found in Appendix \ref{app:proof_for_separate}. Clearly, SAM introduces a perturbation-dependent term that slows down learning across all the features. This slow-down factor turns out to be non-uniform and alters how feature learning is appointed, as we formalize in the next section. Note that in practice, SAM usually decelerates training for the sake of better generalization, as it tracks two gradients at a time, matching the theoretical difference here. 

The ODEs (\ref{eq:SAM_ODE})-(\ref{eq:GD_ODE}) relate the time variable to the evolution of weight $v_i$ for each feature. Since they are both one-dimensional, they admit a closed-form (or a simple integral) solution for relevant training time expressions. This empowers us to quantify how learning compute is allocated across features; in the next section, we use this observation to formalize
simplicity bias and show how SAM reshapes it compared to GD.

\subsection{SAM Has Lower Simplicity Bias than GD}


We now build upon the ODE results and prove that during early training, SAM effectively uniformizes feature learning (i.e. reduces SB) in our multi-head self-attention model. We 
first establish a proxy for quantifying SB, 
and show that the times to complete early feature learning form a more uniform sequence for SAM compared to that of GD.


\textbf{Entropy as a Proxy for Simplicity Bias.}
\label{def:simplicity_bias}
    Let $t_1,\dotsc,t_M$ denote the (feature-wise) times taken by the model to learn $M$ features. Lower simplicity bias means that learning is uniform, where a small subset of features does not consume most of the time. 
    
    To measure this evenness in a scale-free way, we sum-normalize the time sequence and obtain $p_i = t_i / \sum_j t_j, 1\leq i\leq M.$ Viewing this as discrete probability masses, we quantify SB via its Shannon entropy:  
    \allowdisplaybreaks
    \begin{align}\label{eq:sb}
        \ent\left(\left\{t_i\right\}_{i=1}^M\right) & = -\sum_{i=1}^M \frac{t_i}{\sum_j t_j}\log \left(\frac{t_i}{\sum_j t_j}\right) \notag \\ & = -\sum_{i=1}^M p_i\log p_i.
    \end{align}
    A higher entropy indicates a discrete distribution closer to a uniform one and smaller SB towards certain features, since the learning times do not differ too much. 

\textbf{Quantifying SAM's Simplicity Bias in Early Training. }
Theorem~\ref{thm:SAM_ODE} describes an early-training regime where the terms in Eq.~(\ref{eq:SAM_ODE}), (\ref{eq:GD_ODE}) govern the evolution of $v_i(t)$.  We can then take proper integrals to obtain the learning times of each $v_i$ during early training. 
%

Note that the final learned weight $v_i$ in Eq. (\ref{eq:ansatz1}), given a large number of input tokens $N$, is: 
\[
    v_i\ve_i = \left(\lambda_i + \frac{\lambda_i + \Tr(\mSigma)}{N}\right)^{-1/3} \ve_i = O(\lambda_i^{-1/3}) \ve_i. 
\]
This motivates our final assumption that early training brings each feature to a fixed fraction of its eventual magnitude.
\begin{assumption}[Minimal feature learning in early training]\label{assumption:early_training}
    For learning each feature, the evolution of weight $v_i$ admits $c$ (a possibly small constant) s.t. $v_i = c\lambda_i^{-1/3}$ around the end of early training for \textbf{both GD and SAM}. 
\end{assumption}

Assumption \ref{assumption:early_training} implicitly provides integral bounds on $v_i$ so that we can directly integrate the inverse of the ODEs from Theorem \ref{thm:SAM_ODE} and obtain two time sequences. For example, starting with some initialization $v_0$, the time of early training for head $i$, by Theorem \ref{thm:SAM_ODE}, is:
\[
    t_i^{GD} = \int_{v_0}^{c\lambda_i^{-\frac{1}{3}}}\frac{dt}{dv_i} \, dv_i = \int_{v_0}^{c\lambda_i^{-\frac{1}{3}}}\frac{\tau}{\lambda_i^2v_i^2} \, dv_i. 
\]
A similar sequence can be defined for SAM. Now we can use majorization theory~\citep{marshall1979inequalities}, a branch in math that rigorously defines spread and evenness within sequences, to quantify Simplicity Bias via Entropy.  More details can be found in Appendix \ref{subsec:proof_of_SAM_reduces_SB}. The following theorem presents our key result regarding the uniformity in SAM feature learning. 

\begin{figure*}[!t]
    \centering
    \includegraphics[width=0.88\linewidth]{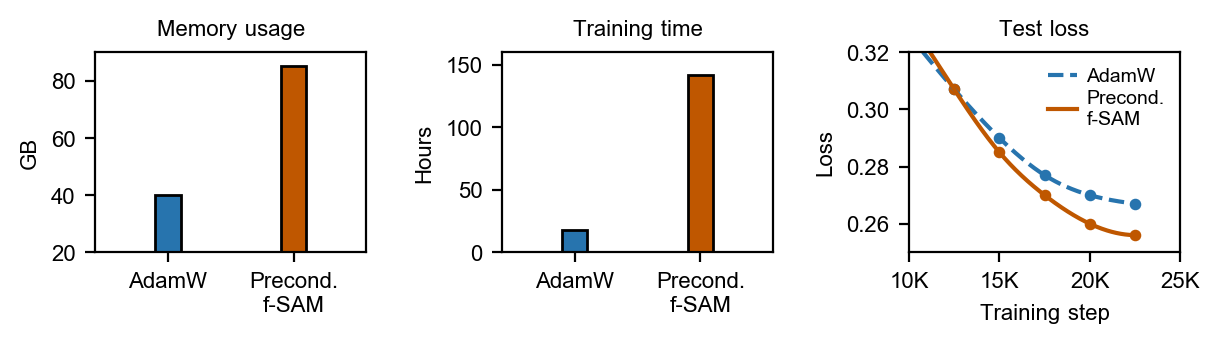}
    \vspace{-3mm}
    \caption{\textbf{Comparison of preconditioned f-SAM and AdamW.}
    f-SAM achieves lower evaluation loss compared to AdamW but incurs around 2x GPU memory usage and 6x training time, when fine-tuning Phi2-2.7B on the MathInstruct dataset. This makes it prohibitively expensive for training even medium-sized LLMs (we couldn't finish fine-tuning with f-SAM after 6 days on 4$\times$ NVIDIA A40 GPUs, hence reporting validation loss).}
    \label{fig:pref_sam_vs_adamw}
    \vspace{-2mm}
\end{figure*}

\begin{theorem}[SAM has lower simplicity bias than GD] \label{thm:SAM_reduces_SB}
     In the setting of Theorem \ref{thm:SAM_ODE} and Assumption \ref{assumption:early_training}, suppose $\texttt{ATTN}_{S}(\mX)$ learns $M$ features in total starting from the same small initializations for both SAM and GD. Each $v_i > 0$ starts at $\varepsilon_i$ and reaches some $c\lambda_i^{-1/3}$ around the end of early training, $i = 1, \dots, M$. Taking $\varepsilon = \max\{\varepsilon_i\}_{i=1}^M$, we integrate over the early training phase ``common" for all the features to get the times: 
\[
   t_i^{\text{GD}}:= \int_{\varepsilon}^{c\lambda_i^{-1/3}}\frac{\tau}{\lambda_i^2v_i^2} \ dv_i, 
\]
\[
    t_i^{\text{SAM}} := \int_{\varepsilon}^{c\lambda_i^{-1/3}}\frac{\tau}{\lambda_i^2v_i^2 - \frac{2\rho}{\sqrt{3}}\lambda_i^2v_i} \ dv_i
\]

Then as long as $\rho < \frac{\sqrt{3}}{2} \varepsilon_i$, 
we have that 
    \[
         \ent\left(\left\{t_i^{\text{SAM}}\right\}_{i=1}^M\right) >  \ent\left(\left\{t_i^{\text{GD}}\right\}_{i=1}^M\right).
    \]
That is, SAM results in lower simplicity bias early in training compared to GD. 
\end{theorem}

\cref{fig:toy_results} (right) empirically verifies that the drops in training loss are more uniform globally and under more general Gaussian initializations (no $v > 0$ restriction). Notably, while our analysis only considers early training, we see that in practice similar results hold throughout training.

\fbox{\parbox{0.47\textwidth}{
\textbf{\textit{Remark}. }
Theorem \ref{thm:SAM_reduces_SB} shows that SAM reduces simplicity bias by making feature learning more homogeneous across directions early in training. This identifies lower SB as a key factor contributing to the superior generalization performance of SAM for training transformer models.
}}

Despite its superior generalization, SAM is prohibitively for even medium sized LLMs. \cref{fig:pref_sam_vs_adamw} shows that when fine-tuning Phi-2 with 2.7B parameters, SAM yields a lower test loss in expense of over 2x memory and over 6x training time. 
This motivates our next section, which proposes to reduce SB during training with gradient-based optimizers, e.g. AdamW and Muon, by altering the training data distribution.\looseness=-1

\section{A Data-centric Approach to Reduce Simplicity Bias }\label{sec:method}


In this section, 
we leverage our theoretical insight to propose a light-weight data-centric algorithm for lowering simplicity bias during training to boost the generalization performance.

In language modeling, a single example can often be predicted using multiple features, such as lexical memorization (specific words or phrases), syntactic templates, stylistic cues, topic-level correlations, or shallow semantic associations. As soon as the model captures one such feature—e.g., recognizing a familiar pattern or phrase—the token-level predictions improve and the loss drops, even if deeper semantic or reasoning-related features remain unlearned. As a result, low loss implies that 
the model has latched onto the easiest available feature, consistent with our theoretical results confirming simplicity bias for transformer models.

\textbf{Identifying easy vs difficult examples.} To identify examples containing at least one easy feature that is learned early by the model, we train a smaller proxy model for a small number of iterations (e.g. $1/3$ of an epoch) and track the loss trajectory—loss values at a few checkpoints during training the proxy model—for all examples. We then identify examples containing easy features as those with steadily decreasing loss trajectories or having a low loss value at the last checkpoint. We separate easy vs difficult examples by clustering the loss trajectories or loss values into two groups. While both approaches are effective, we show in our experiments that relying on loss trajectories is more robust and yields a slightly higher performance.
Because the proxy is trained briefly and can be much smaller than the target, this method is lightweight and can be applied as a preprocessing step. 


\textbf{Making feature learning more uniform by amplifyng difficult features.} 
To reduce simplicity bias, we aim to amplify the features in the difficult cluster by upsampling them or synthetically rephrase them using an LLM. 
Training the target model on the upsampled data allows the underlying  optimizer to learn easy and difficult features at a more uniform speed.
%
%
As illustrated in Figure \ref{fig:toy_results}, upsampling examples in the difficult cluster induces drops corresponding to learning a new feature more uniform over time and improves generalization indicated by the lower validation loss.



\section{Experimental Setup}
\label{sec:setup}
In this section, we evaluate the effectiveness of our data centric approach in Section \ref{sec:method}.
%

\begin{figure*}[!ht]
    \centering
    \includegraphics[width=0.98\linewidth]{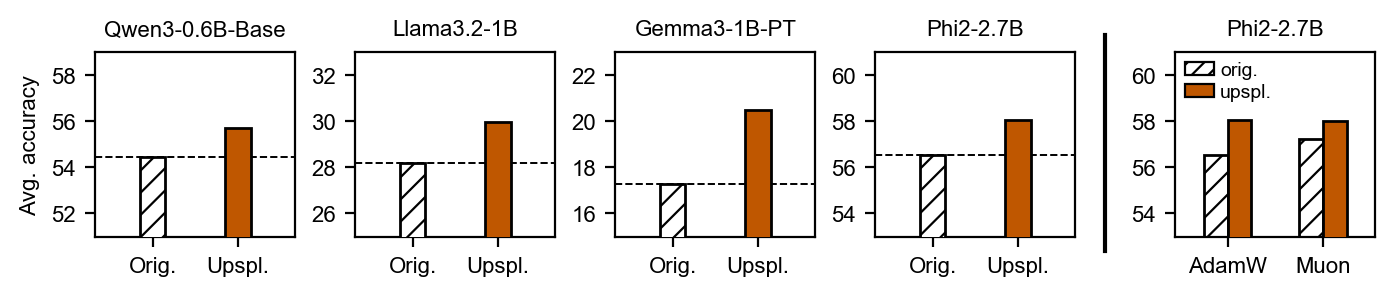}
    \vspace{-1mm}
    \caption{
        \textbf{Upsampling improves performance across models and math benchmarks.}
        (Left) Each model is finetuned with AdamW on the original dataset and \textit{with} targeted upsampling, and evaluated using zero-shot greedy decoding.
        (Right) Targeted upsampling is optimizer-agnostic, yielding consistent performance gains under Muon.
        For each model, we report the final-epoch checkpoint with the highest average accuracy across all benchmarks.
        Detailed dataset-level results are provided in \cref{tab:adamw-math-upsampling}, Appendix~\ref{app:aux_results}.
        (Different scales are used across subfigures to accurately show performance gains while accounting for variability in individual model performance.)
    }
    \label{fig:main_results}
\end{figure*}

\textbf{Dataset.}
For finetuning, we use the MathInstruct dataset \citep{yue2023mammoth}, a large and diverse math corpus containing over $262,000$ instruction-response pairs, spanning a wide variety of mathematical topics with varying difficulty.
%
%
The dataset is aggregated from $13$ heterogeneous and {imbalanced} sources, resulting in substantial variability in task formatting, solution style, and problem difficulty.
%


\textbf{Finetuning.}
We consider a range of \textit{pretrained} LLMs at different model sizes, including Qwen3-0.6B-Base \citep{yang2025qwen3}, Llama3.2-1B \citep{dubey2024llama}, Gemma3-1B-PT \citep{team2025gemma}, and Phi2-2.7B \citep{javaheripi2023phi}.
We follow a model setup similar to \citet{nguyen2024mini}.
Specifically, all the models are finetuned using trainable LoRA adapters \citep{hu2022lora} with rank $128,$ scaling factor $\alpha = 512,$ and dropout $0.05.$
For Phi2, LoRA weights are added to both the self-attention (query, key, value projections) and feed-forward modules.
For all other models, LoRA weights are only added to the self-attention modules (query, key, value, and output projections).
All models are finetuned on MathInstruct, with a maximum context length of $512$ tokens, an effective batch size of $128,$ for a total of three epochs, using AdamW or Muon with a cosine learning rate scheduler (maximum learning rate $2\times10^{-5}$ and warmup ratio $0.03$).

To identify examples for upsampling, we use Pythia-70M \citep{biderman2023pythia} as a lightweight proxy model. 
%
This provides a compute-efficient way for identifying under-learned examples without finetuning each larger target model.
%
%
%
Loss trajectories are collected based on $5$ checkpoints saved during training the proxy model for $1/3$ epoch on MathInstruct.
%
%
%
%
We separate easy vs difficult examples by clustering the loss trajectories or final loss values, and upsample the difficult cluster as detailed in Sec. \ref{sec:method}.


\textbf{Evaluation}
Consistent with \citet{yue2023mammoth}, we report the accuracies of models finetuned on the original and upsampled datasets across a variety of math test benchmarks, including GSM8K \citep{cobbe2021training}, MATH \citep{hendrycks2021measuring}, NumGLUE \citep{mishra2022numglue}, SVAMP \citep{patel2021nlp}, DeepMind \citep{davies2021advancing}, and MMLU-Math \citep{hendrycks2020measuring}.
Together, these benchmarks capture diverse aspects of model generalization.
%
%
%
Across all benchmarks, models are evaluated in a zero-shot setting, using greedy decoding with a maximum generation length of $2,048$ tokens, which is sufficient, given that models are not prompted for long, chain-of-thought reasoning.

\begin{figure*}[!ht]
    \centering
    \includegraphics[width=\linewidth]{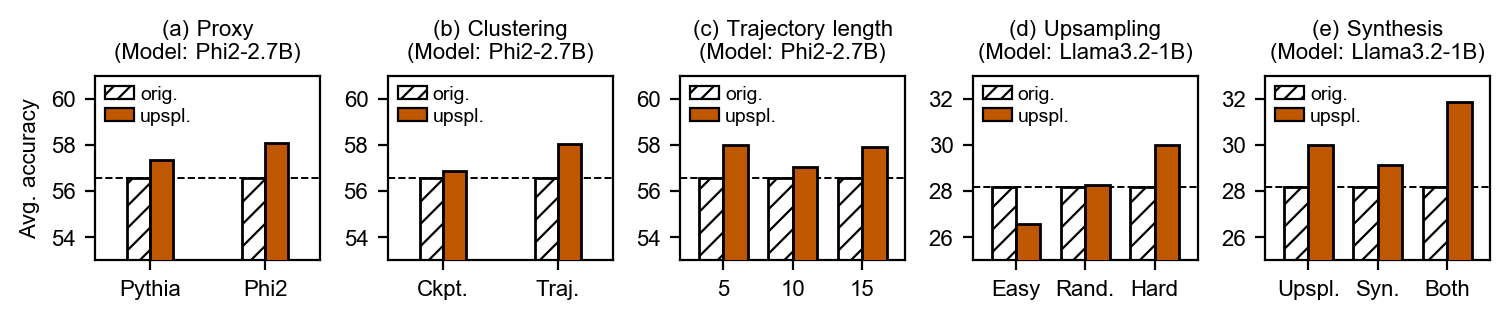}
    \vspace{-4mm}
    \caption{
        \textbf{Data upsampling ablations.}
        Each subfigure lists the target model in its title; detailed dataset- and model-level results are provided in Appendix~\ref{app:aux_results}.
        (a) Performance is comparable when using the lightweight Pythia-70M as the proxy instead of the target model.
        (b) Identifying \textit{hard} examples using loss trajectories from all $15$ proxy checkpoints performs better than using only the final checkpoint.
        (c) No clear trend emerges as a function of the number of proxy checkpoints used in loss trajectories for clustering.
        (d) Under compute-matched finetuning, upsampling hard examples yields the strongest generalization, while upsampling easy examples degrades performance.
        (e) Variational problem synthesis offers a promising strategy to further improve model generalization on hard examples.
    }
    \label{fig:ablations}
    \vspace{-1mm}
\end{figure*}

\section{Results}
\label{sec:results}


\cref{fig:main_results} reports accuracies for different models finetuned on the original and upsampled datasets across the selected math benchmarks.
Using Pythia-70M as the proxy model, we identify approximately $38\%$ of the training examples as hard, which are then selected for upsampling.
Consistent with our theoretical findings, finetuning \textit{with} targeted upsampling yields consistent performance improvements across all models, compared to finetuning on the original dataset, with gains of up to $\percentimpr{17.3}{20.5}\%.$ 
Moreover, targeted upsampling is largely optimizer-agnostic, producing performance gains under both AdamW and Muon. 
%
Detailed, dataset-level results are in \cref{tab:adamw-math-upsampling,tab:adamw-upsampling-variants} of Appendix~\ref{app:aux_results}.

Additionally, as can be seen in \cref{fig:main_results}, lower-performing models, those often limited in capacity to learn from complex datasets, show the largest gains, suggesting that targeted upsampling amplifies learning where it is most needed, without requiring larger models or additional \textit{new} data.
%
%


\subsection{Ablations}

To better understand what factors contribute to the observed performance gains with targeted upsampling, we conduct a series of controlled ablations, shown in \cref{fig:ablations}.



\textbf{Choice of the Proxy Model.}
Using a small proxy model (Pythia-70M) across all target models offers a compute-efficient way of collecting loss trajectories for identifying hard examples.
However, one might expect that using the target model itself to compute loss trajectories could be advantageous, as it more directly captures the model's inductive biases.
We examine this tradeoff by comparing the performance of Phi-2.7B finetuned with upsampling based on loss trajectories derived from Pythia-70M against those derived from Phi-2.7B itself.
As shown in \cref{fig:ablations}(a), performance is comparable in both cases, despite Pythia-70M upsampling $\percentdecr{138806}{99273}\%$ \textit{fewer} examples than Phi2-2.7B.
%
Effectively, Pythia-70M learns easy features faster than Phi2-2.7B, due to its smaller size. This results in a larger cluster of easy examples. Nevertheless, upsampling the smaller more difficult cluster is sufficient for improving the generalization performance, confirming our theoretical results and the effectiveness of smaller proxy model.



\textbf{Clustering: Loss Trajectory vs. Checkpoint.}
Next, we compare using per-example loss trajectories computed over $5$ proxy checkpoints (corresponding approximately to $1/3$ epoch of finetuning) against using losses from only the final ($5$th) checkpoint.
As shown in \cref{fig:ablations}(b), using the full loss trajectories outperforms the last checkpoint losses.
This suggests that the loss signals from \textit{early} training stages better capture example difficulty, than a single, late-stage checkpoint.




\textbf{Loss Trajectory Length.}
Relatedly, we investigate how the \textit{number} of proxy checkpoints used in constructing loss trajectories affects performance.
Specifically, we compare upsampling based on loss trajectories computed from $5, 10,$ and $15$ checkpoints, corresponding to approximately $1/3, 2/3,$ and one full epoch of proxy finetuning, respectively.
%
Although, the number of hard examples identified varies across these settings (ranging from $38$--$44\%$ of the dataset), \cref{fig:ablations}(c) shows no consistent performance trend as the loss trajectory length increases.
That said, the results using first $5$ checkpoints further corroborate our earlier observation that loss signals from early training stages better capturing example difficulty.



\textbf{Upsampling Easy vs Difficult vs Random Examples.}
To determine whether the observed gains stem from \textit{additional} training steps induced by upsampling, we evaluate performance 
%
when finetuning models on datasets upsampled with the same number of easy, random, or hard examples. 
As shown in \cref{fig:ablations}(d), random upsampling yields performance comparable to finetuning without upsampling, confirming that the improved generalization is not due to a larger number of training steps.
%
In contrast, (unsurprisingly,) upsampling easy examples downgrades model performance, while upsampling hard examples showed the strongest generalization.
This further reinforce our theoretical findings that upsampling of hard examples promotes uniform feature learning, and 
boosts generalization.



\textbf{Variational Problem Synthesis.}
Finally, we investigate whether upsampling can benefit from \textit{rephrasing} hard examples, rather than, or in addition to, simply duplicating them.
To this end, we consider synthetic variants of the hard examples, generated using their original solutions.
As shown in \cref{fig:ablations}(e), while synthetic examples improve performance, they still fall short of the gains achieved by simple duplication.
Further analysis revealed that only about $50\%$ of the synthetic generations were clustered as being hard.
Due to our compute constraints, we did not pursue additional prompt tuning or rejection sampling.
Nevertheless, we positively report that combining synthetic examples with duplication yields stronger gains than either approach alone, making variation problem synthesis a promising direction for future study.

\section{Conclusion}
\label{sec:concl}
In this work, we theoretically analyzed an in-context linear regression model with multi-head linear self-attention, when trained with gradient descent and Sharpness-Aware-Minimization (SAM). We identified, for the first time, that lower simplicity bias of SAM is a key contributor to its superior generalization performance. While SAM is computationally prohibitive for LLMs, we demonstrated that upsampling or augmenting examples learned later in training reduce the simplicity bias during training and yields improved generalization performance. Extensive experiments across multiple LLM architectures and optimizers confirm that this strategy consistently improves mathematical reasoning performance. Together, our results highlight data distribution as a powerful and practical lever for shaping optimization dynamics and improving generalization in large-scale language model training.
%



\section*{Impact Statement}

This paper presents work whose goal is to advance the field of Machine
Learning. There are many potential societal consequences of our work, none
which we feel must be specifically highlighted here.


\bibliography{references}
\bibliographystyle{conf/iclr2026_conference}


\newpage\onecolumn

\begin{appendices}
    \section{Equivalence between Stylized Attention and Neural Networks} \label{app:equivalence}

\subsection{Proof Techniques~\citep{zhang2024trained}} \label{subsec:proof_of_equivalence}
In this section, we briefly discuss the proof techniques from~\citep{zhang2024trained} and reiterate some results regarding the reduction of multi-head self-attention models into neural networks. First, recall the separate key-query configuration: 
\[
    \texttt{ATTN}_S(\mX)   = \mX + \sum_{i=1}^H\frac{1}{N}\mW_i^V\mX\mX^\top{\mW_i^{K}}^\top\mW_i^{Q}\mX, \quad \text{with} \ \ \mW_i^V=\begin{bmatrix} 
            * & *\\ 
            \vv_i & v_i
        \end{bmatrix}, \ 
        \mW_i^K=\begin{bmatrix} \vk_i^\top & k_i \end{bmatrix}, \ 
        \mW_i^Q=\begin{bmatrix} \vq_i^\top & q_i \end{bmatrix}. 
\]
Here (*) denotes sub-matrices with appropriate dimensions whose values do not matter at all. Given $\hat{y}_q = \texttt{ATTN}_S(\mX)_{d+1, N+1}$ the predicted response for the query and the loss $\gL= \E[(y_q-\hat{y}_q)^2]$, Appendix F.1 of~\citep{zhang2024context} justifies the choice of zero blocks by showing that the gradients w.r.t $k_i$ and $\vv_i$ equal $0$, meaning that they remain at initializations throughout training. Hence, we directly set the initialization to be $0$ for simplicity.

Now we can plug in the weights and it follows that: 
\begin{align*}
    \texttt{ATTN}_S(\mX)  & = \mX + \sum_{i=1}^H\frac{1}{N}\mW_i^V\mX\mX^\top{\mW_i^{K}}^\top\mW_i^{Q}\mX \\
    & = \mX + \frac{1}{N}\sum_{i=1}^H\begin{bmatrix} 
            * & *\\ 
            \vzero & v_i
        \end{bmatrix}
        \begin{bmatrix} 
            \sum_{i=1}^N\vx_i\vx_i^\top + \vx_q \vx_q^\top& \sum_{i=1}^Ny_i\vx_i\\ 
            \sum_{i=1}^Ny_i\vx_i^\top & \sum_{i=1}^Ny_i^2
        \end{bmatrix}\begin{bmatrix}
            \vk_i \\ 0
        \end{bmatrix}\begin{bmatrix} \vq_i^\top & * \end{bmatrix}\mX \\ 
        & = \mX + \frac{1}{N}\sum_{i=1}^H\begin{bmatrix} 
            * & *\\ 
            v_i\sum_{i=1}^Ny_i\vx_i^\top & v_i\sum_{i=1}^Ny_i^2
        \end{bmatrix}
        \begin{bmatrix}
            \vk_i\vq_i^\top & * \\
            \vzero & 0
        \end{bmatrix}\mX \\
        &= \begin{bmatrix}
     \vx_1&  \vx_2& \dots&  \vx_N & \vx_q \\
      y_1& y_2&  \dots&  y_N & 0 
    \end{bmatrix} + \frac{1}{N}\sum_{i=1}^H\begin{bmatrix} 
            * & *\\ 
            v_i\sum_{i=1}^Ny_i\vx_i^\top\vk_i\vq_i^\top & *
        \end{bmatrix}
    \begin{bmatrix}
     \vx_1&  \vx_2& \dots&  \vx_N & \vx_q \\
      y_1& y_2&  \dots&  y_N & 0 
    \end{bmatrix}.
\end{align*}
We take the bottom right entry and have: 
\begin{align*}
    \hat{y}_q = \texttt{ATTN}_S(\mX)_{d+1, N+1} = 0 + \frac{1}{N}\sum_{i=1}^H\sum_{i=1}^Ny_i\vx_i^\top v_i\vk_i\vq_i^\top\vx_q. 
\end{align*}
Let $\texttt{vec}(\cdot)$ denote the column-wise flattening of a matrix into a column vector. Define the following cubic map of the input $\mX$, and algebra shows that: 
\begin{equation} \label{eq:z}
    \vz = \vz(\mX) = \texttt{vec}\left(\frac{1}{N}\sum_{i=1}^Ny_i\vx_i\vx_q^\top\right) \in \mathbb{R}^{d^2} \implies \hat{y}_q = \sum_{i = 1}^Hv_i\vq_i^\top\mK_i\vz, \quad \text{with } \ \ 
        \mK_i = \begin{bmatrix}
            \vk_i^\top  & \vzero_d^\top  & \dots & \vzero_d^\top \\
            \vzero_d^\top & \vk_i^\top & \dots & \vzero_d^\top \\ 
            \vdots & \vdots & \ddots & \vdots \\ 
            \vzero_d^\top & \vzero_d^\top & \dots & \vk_i^\top \\ 
        \end{bmatrix} \in \mathbb{R}^{d \times d^2}. 
\end{equation}
It is also common in theory to consider a simpler model, where the key and query matrices are merged into one single matrix that aims to simplify the analysis further: 
\[
    \texttt{ATTN}_M(\mX)  = \mX + \sum_{i=1}^H\frac{1}{N}\mW_i^V\mX\mX^\top{\mW_i^{KQ}}\mW_i^{Q}\mX, \quad \text{where} \ \ \mW_i^V=\begin{bmatrix} 
            * & *\\ 
            \vv_i & v_i
        \end{bmatrix}, \ \ \mW_i^{KQ}=\begin{bmatrix} 
            \mU_i & *\\ 
            \vu_i^\top & *
        \end{bmatrix}.
\]
While we focus on the separate configuration, we provide some discussion in the appendices. First, by repeating the similar analysis as above, ~\citep{zhang2025training} show that we can initialize $\vv_i, \vu_i = \vzero$ and subsequently, some algebra shows that
\[
    \texttt{ATTN}_M(\mX)_{d+1, N+1} = \vw_2^\top\mW_1\vz, \quad \text{with} \ \ \vw_2 = \begin{bmatrix}
        v_1 \\ v_2 \\ \vdots \\ v_H
    \end{bmatrix} \in \mathbb{R}^H, \ \ \mW_1 = \begin{bmatrix}
        \texttt{vec}(\mU_1) \\ \texttt{vec}(\mU_2) \\ \vdots \\ \texttt{vec}(\mU_H)
    \end{bmatrix}, 
\]
which is a two-layer fully-connected linear neural network. This makes sense intuitively, as merging key and query together loses information and can no longer admit a third layer. In the following section, we provide the manifolds of fixed points in the training dynamics of each. 

\subsection{Manifolds of Fixed Points} \label{subsec:manifolds}

We directly reiterate the conclusions from~\cite{zhang2025training}, in particular Section 3.2 and Appendix E.3, with slightly modifed notations. 

 For $\texttt{ATTN}_M(\mX)_{d+1, N+1}$, there are two manifolds of fixed points, corresponding to vanishing initialization and final learning: 
\[
    \gM_0 = \{\vw_2 = \bm{0}, \mW_1 = \bm{0}\}, \quad \gM_{*} = \left\{\vw_2, \mW_1 : \vw_2^\top\mW_1 = \mathbb{E}\left[y_q\vz^\top\right]\mathbb{E}\left[\vz\vz^\top\right]^{-1}\right\}. 
\]
Here learning is achieved when the model weights move from small initializations (around the zero saddle) to the other stable manifold, resulting in a sharp drop in training loss. 

For $\texttt{ATTN}_S(\mX)_{d+1, N+1}$ instead, there are $2^d$ manifolds of fixed points. Let the covariance $\mSigma = \sum_{i=1}^d\lambda_i\ve_i\ve_i^\top$ ordered by the eigenvalues. A manifold exists for each subset of learned eigenvectors with indices $\gS_m \subseteq \{1,2, \dots, d\}\ \text{ where } m = 0, \dots, d, \ \ |\gS_m| = m$, and the manifold of fixed points for this set is given by: 
\begin{align*}
    \gM(\gS_m) = & \biggl\{v_i, \vq_i, \mK_i : i = 1, \dots, H; \ \text{such that the following three conditions hold: }\\
    & \ \text{(1)}  \sum_{i=1}^H v_i \vk_i \vq_i^\top = \sum_{j \in \gS_m}\left(\lambda_j + \frac{\lambda_j + \Tr(\mSigma)}{N}\right)^{-1}\ve_j\ve_j^\top;\\
    & \ \text{(2)}\  v_i \neq 0 \implies \text{both } \  \vk_i, \ \vq_i \in \text{Span} \left\{\ve_j\right\}_{j \in \gS_m}; \\
    & \ \text{(3)}\  v_i = 0 \implies \text{at least one of } \  \vk_i, \ \vq_i \in \text{Span} \left\{\ve_j\right\}_{j \in \gS_m} \biggl\}. 
\end{align*}
The precise derivation of these conditions come from an equality constrained optimization problem. However, the key takeaway here is that the ansatz (\ref{eq:ansatz1})-(\ref{eq:ansatz3}) satisfies the above conditions and serves as a reasonable educated guess to facilitate analyzing the highly nonlinear training dynamics. Here learning is done progressively by learning one eigenvector at a time, in the order of decreasing eigenvalues. We observe saddle-to-saddle learning dynamics, where the model moves from a saddle in $\gM(\text{top $i$ eigenvector indices})$ to a saddle in $\gM(\text{top $i+1$ eigenvector indices})$. 

\textbf{Remark. } The latter observation of progressively learning the eigenvalues from largest to smallest reflects the idea of simplicity bias, i.e., easy (more obvious) features are learned before difficult (less intuitive) ones. However, a practical concern is that if the model gets stuck at one saddle point for too long, it might reach convergence before getting to learn more complex features. As a result, reducing simplicity bias can help the model ``escape" earlier and explore more.
    \section{Preliminary Investigation into Merged Key-Query Matrix}

In this part of the appendix, we provide an attempt to study the simpler setting of merged key-query configuration $\texttt{ATTN}_M$ with a SAM variant suing similar techniques. From Appendix \ref{subsec:manifolds}, this model has two manifolds of fixed points, one around small initializations, and escaping into the other one corresponds to the drop in training loss. We establish early training ODEs and insights on the dynamics, similar to the $\texttt{ATTN}_S$ results. 

\subsection{Necessary Gradients \& Hessians} \label{subsec:merged_grad_hess}
Recall that we study the expected squared loss for the predicted label of a new query $\gL = \mathbb{E}[(y_q - \hat{y}_q)^2]$. In particular, for joint key-query matrix, $y_q$ is the output from a fully connected two-layer neural network, and the loss becomes: 
\[
    \gL = \mathbb{E}\left[\left(y_q - \vw_2^\top\mW_1\vz\right)^2\right]. 
\]
This gives us the following gradient and Hessian with respect to $\vw_2$ and $\mW1$: 
\begin{align*}
    & \frac{\partial \gL}{\partial \vw_2} =\nabla_{\vw_2}\gL = -2\mathbb{E}\left[\left(y_q - \vw_2^\top\mW_1\vz\right)\left(\mW_1\vz\right)\right] = -2\mW_1\left(\mathbb{E}\left[y_q\vz^\top \right] - \vw_2^\top\mW_1\mathbb{E}\left[\vz\vz^\top\right]\right)^\top. \\
    & \mH(\vw_2) =\nabla_{\vw_2}^2\gL  = 2\mW_1\mathbb{E}\left[\vz\vz^\top\right]\mW_1^\top.  \\ 
    & \frac{\partial \gL}{\partial \mW_1} =\nabla_{\mW_1}\gL = -2\mathbb{E}\left[\left(y_q - \vw_2^\top\mW_1\vz\right)\left(\vw_2\vz^\top\right)\right] = -2\vw_2\left(\mathbb{E}\left[y_q\vz^\top \right] - \vw_2^\top\mW_1\mathbb{E}\left[\vz\vz^\top\right]\right).  
\end{align*}
For Hessian matrix, it makes more sense to consider $\vw_1 = \texttt{vec}(\mW_1) \in \mathbb{R}^{d^2}$ flattened into a vector. Let $\otimes$ denote the Kronecker matrix product. We then have
\begin{align*}
    \frac{\partial \gL}{\partial \texttt{vec}(\mW_1)} =\nabla_{\vw_1}\gL & = -2\texttt{vec}\left(\vw_2\mathbb{E}\left[y_q\vz^\top \right]\right) - \texttt{vec}\left(\vw_2\vw_2^\top\mW_1\mathbb{E}\left[\vz\vz^\top\right]\right) \\ 
    & =  -2\texttt{vec}\left(\vw_2\mathbb{E}\left[y_q\vz^\top \right]\right) -2\left(\mathbb{E}\left[\vz\vz^\top\right] \otimes (\vw_2\vw_2^\top)\right)\texttt{vec}\left(\mW_1\right) \\ 
    & = -2\texttt{vec}\left(\vw_2\mathbb{E}\left[y_q\vz^\top \right]\right) -2\left(\mathbb{E}\left[\vz\vz^\top\right] \otimes (\vw_2\vw_2^\top)\right)\vw_1. \\  
    \mH(\texttt{vec}(\mW_1)) = \nabla_{\vw_1}^2\gL &= 2\left(\mathbb{E}\left[\vz\vz^\top\right] \otimes (\vw_2\vw_2^\top)\right). 
\end{align*}

To simplify the analysis, we study the merged key-query configuration under layer-wise SAM defined below. We argue in the subsequent Section \ref{subsec:proof_for_merged_white_cov} that it derives similar insights. 

\begin{definition} [Global vs. layer-wise SAM regimes] For model weights $\vw= (\vw_1, \vw_2, \dots, \vw_k)$ of $k$ groups/layers, we consider the following two SAM regimes: 
\begin{align*}
     &\textbf{Global SAM: } \text{All the parameters are updated at once via Equation \ref{eq:GD_and_SAM_updates}. }  \\
    &\textbf{Layer-wise SAM: } \text{Each layer of parameters is updated separately via $k$ SAMs: } \ \ \tau\dot{\vw_i} = -\frac{1}{2}\nabla_{\vw_i}\gL\left(\vw_i +\frac{\rho \nabla_{\vw_i}\gL}{\|\nabla_{\vw_i}\gL\|_2}\right). 
\end{align*} 
\end{definition}
Note that in the main paper, we consider $\texttt{ATTN}_S$ with the practical global SAM (More precisely, since the ansatz guarantees that only the current active head has dynamics, the other heads remain unchanged and do not contribute any gradient. Hence, layer-wise SAM on the active head/layer is equivalent to global SAM). 

\subsection{Proof of Lemma \ref{lem:static_saddle_points} and Stability of Saddle Landscape}
Under the assumption that $\sigma_{\text{min}}\left(\mI + \mP\right) > \bm{0}$ during gradient flow i.e. the perturbation never cancels out the raw gradient, the lemma statement is trivial from the approximation \ref{eq:SAM_first_order}: 
\[
    \tau\dot{\mW}_{\text{SAM}} \approx (\mI + \mP)\tau\dot{\mW}_{GD}, 
\]
where $\mI + \mP$ is full rank. Theoretically, this allows us to isolate the observation of reduced SB to the SAM optimizer, instead of distribution of saddles or significant change of loss landscape.

\subsection{Training Dynamics ODE for White Covariance} \label{subsec:training_dynamics_ODE_for_white} \label{subsec:proof_for_merged_white_cov}

A similar conservation law (ansatz) among the weights can be established for the two-layer fully connected neural network equivalence $\texttt{ATTN}_M$ (see~\citep{fukumizu1998effect, du2018algorithmic, saxe2019mathematical, ji2018gradient, zhang2025training}) that derives from vanishing initializations: 
\[
    \vw_2\vw_2^\top - \mW_1\mW_1^\top \approx \bm{0} \implies \|\vw_2\|_2 \approx \|\mW_1\|_F = \|\texttt{vec}(\mW_1)\|_2. 
\]
Suppose that the perturbation radius $\rho$ is sufficiently small in the sense that it preserves this conservation law. Since the two layers of weights maintain equal norms and hence evolve similarly, under the additional ansatz that their gradient norms $\|\nabla_{\vw_2}\gL\|_2$ and $\|\nabla_{\mW_1}\gL\|_F$ are of the same order, global SAM update now depends on the perturbation: let $\vw = (\texttt{vec}(\mW_1)^\top, \vw_2^\top)^\top = (\vw_1^\top, \vw_2^\top)^\top$, 
\begin{align*}
    \nabla_{\vw}\gL\left(\vw + \rho\frac{\nabla_{\vw}\gL(\vw)}{\|\nabla_{\vw}\gL(\vw)\|_2}\right) & = \nabla_{\vw}\gL\left(\begin{bmatrix}
        \vw_1 \\ \vw_2
    \end{bmatrix} + \frac{\rho}{\sqrt{\|\nabla_{\vw_1}\gL\|_2^2 + \|\nabla_{\vw_2}\gL\|_2^2}}\begin{bmatrix}
         \nabla_{\vw_1}\gL \\ \nabla_{\vw_2}\gL
    \end{bmatrix}\right) \\
    & = \nabla_{\vw}\gL\left(\begin{bmatrix}
        \vw_1 \\ \vw_2
    \end{bmatrix} + \frac{\rho}{\sqrt{C\|\nabla_{\vw_2}\gL\|_2^2 + \|\nabla_{\vw_2}\gL\|_2^2}}\begin{bmatrix}
         \nabla_{\vw_1}\gL \\ \nabla_{\vw_2}\gL
    \end{bmatrix}\right) \ \ \text{for constant $C$}. 
\end{align*}
We focus on the update on $\vw_2$: 
\[
   \nabla_{\vw_2}\gL\left(\vw_2 + \frac{\rho}{\sqrt{(C+1)\|\nabla_{\vw_2}\gL\|_2^2}}\nabla_{\vw_2}\gL\right) = \nabla_{\vw_2}\gL\left(\vw_2 + \frac{\hat{\rho}\nabla_{\vw_2}\gL}{\|\nabla_{\vw_2}\gL\|_2}\vw_2\right), \ \text{where } \hat{\rho} = \frac{\rho}{\sqrt{C+1}}. 
\]
A similar argument holds for $\vw_1$. Hence, layer-wise SAM becomes similar to global SAM up to a scaled parameter $\hat{\rho}$. This would partially justify how layer-wise SAM can potentially yield similar insights.

We now focus on the setting of white covariance i.e. $\mathbb{E}[\vz\vz^\top] = \alpha \mI_{d^2}$; recall that $\vz$ (Eq. \ref{eq:z}) is a cubic feature map on the input $\mX$. Similarly, due to the weight conservation ansatz, the evolution of $\vw_2$ captures global dynamics. 

\begin{theorem}[SAM accelerates training for merged KQ] \label{thm:merged_SAM_ODE} Under sufficiently small initializations for $\texttt{ATTN}_M$, suppose that SAM perturbation preserves the weight conservation $\vw_2\vw_2^\top - \mW_1\mW_1^\top \approx 0$ and can be well approximated by the first-order expansion \ref{eq:SAM_first_order}. Let $s = \vw_2^\top\vw_2 = \|\mW_1\|_F^2$, $\gamma = \|\mathbb{E}[y_q\vz]\|_2$, and $\mathbb{E}[\vz\vz^\top] = \alpha \mI$ i.e. white input data covariance. Then the evolution of the weights $s(t)$ when we perform layer-wise SAM on $\vw_2$ satisfies: 
\begin{align*}
   &\frac{t}{\tau} =\frac{1}{2\gamma}\ln\frac{s(t)}{s(0)} +  \frac{1}{\alpha\sqrt{\rho^2 + \frac{4\gamma}{\alpha}}}\left( \frac{1}{r_-}\ln\frac{\sqrt{s(t)}-r_-}{\sqrt{s(0)}-r_-}  -\frac{1}{r_+}\ln\frac{\sqrt{s(t)}-r_+}{\sqrt{s(0)}-r_+}\right) 
   , \ \ \text{with} \  \ r_\pm = \frac{\rho \pm \sqrt{\rho^2 + \frac{4\gamma}{\alpha}}}{2}.
\end{align*}
In particular, this equation is the implicit closed-form solution to the following SAM dynamics (when $\rho = 0$, the first term recovers GD only dynamics in \cite{zhang2025training}): 
\[
    \tau \dot{s} = 2s(\gamma - \alpha s)  + 2\rho\alpha s^{3/2}, \quad \text{where } \ \gamma - \alpha s > 0.
\]
By convention, we let $\tau \dot{s} = 0$ when $s = \gamma/\alpha$ i.e. the weight increases and reaches equilibrium. 
\end{theorem}

Note that the dynamics has an unstable equilibrium at $s = 0$ and a stable one at $s = \gamma/\alpha$. In the continuous flow, small initializations start near the former and move to the latter via this ODE, where the training loss sharply decreases. In particular, the SAM dynamics implements a boosting term $2\rho\alpha s^{3/2}$ to allow the weights to evolve faster. 

\begin{proof}
According to~\citep{saxe2013exact}, the conservation law implies: 
\begin{equation} \label{eq:change_of_variable}
    \mW_1 = \vw_2\frac{\mathbb{E}[y_q\vz^\top]}{\|\mathbb{E}[y_q\vz]\|_2} = \vw_2\vm^\top, \quad \text{where } \vm =  \frac{\mathbb{E}[y_q\vz^\top]}{\|\mathbb{E}[y_q\vz]\|_2} \text{ is a unit vector.}
\end{equation}
which gives rise to the following ODE that describes the GD dynamics in \cite{zhang2025training}: 
\[
    \tau(\dot{\vw}_{2})_\text{GD} = \vw_2(\|\mathbb{E}[y_q\vz]\|_2 - \alpha\vw_2^\top\vw_2) = -\frac{1}{2}\nabla_{\vw_2}\gL. 
\]
With the existing theory, we are now ready to derive the approximate SAM dynamics. We start with the perturbation matrix $\mP$ using formulas in Appendix \ref{subsec:training_dynamics_ODE_for_white}: 
\begin{align*}
    \mP = \frac{\rho}{\|\nabla_{\vw_2}\gL\|_2}\mH(\vw_2) & = \frac{\rho}{2\left\Vert\mW_1\left(\mathbb{E}\left[y_q\vz^\top \right] - \vw_2^\top\mW_1\mathbb{E}\left[\vz\vz^\top\right]\right)^\top\right\Vert_2}\left(2\mW_1\mathbb{E}\left[\vz\vz^\top\right]\mW_1^\top\right) \\ 
    & = \frac{\rho}{\left\Vert\vw_2\vm^\top\left(\mathbb{E}\left[y_q\vz^\top \right] - \alpha\vw_2^\top\vw_2\vm^\top\right)^\top\right\Vert_2}\alpha\left(\vw_2\vm^\top\vm\vw_2^\top\right) \\ 
    & = \frac{\rho\alpha\vw_2\vw_2^\top}{\left\Vert\vw_2\right\Vert_2\left|\|\mathbb{E}[y_q\vz]\|_2 - \alpha\vw_2^\top\vw_2\right|} \quad \text{by } \ \vm = \frac{\mathbb{E}[y_q\vz^\top]}{\|\mathbb{E}[y_q\vz]\|_2}.  
\end{align*}

Subsequently, we can approximate the SAM dynamics via: 
\begin{align*}
     \tau(\dot{\vw}_{2})_\text{SAM} = (\mI + \mP)\tau(\dot{\vw}_{2})_\text{GD} = (\mI + \mP)\vw_2(\|\mathbb{E}[y_q\vz]\| - \alpha\vw_2^\top\vw_2). 
\end{align*}

To construct a solvable one-variable ODE, we consider $s = \vw_2^\top\vw_2$, which measures the evolution of weights via the norm, and let $\gamma = \|\mathbb{E}\left[y_q\vz \right]\|_2$: 
\begin{align*}
     \tau(\dot{s})_\text{SAM} = 2\vw_2^\top\tau(\dot{\vw}_{2})_\text{SAM} & = 2\vw_2^\top(\mI + \mP)\vw_2(\|\mathbb{E}[y_q\vz]\|_2 - \alpha\vw_2^\top\vw_2) \\
     & = 2\vw_2^\top\left(\mI + \frac{\rho\alpha\vw_2\vw_2^\top}{\left\Vert\vw_2\right\Vert_2\left|\gamma - \alpha\vw_2^\top\vw_2\right|}\right)\vw_2(\gamma - \alpha\vw_2^\top\vw_2) \\ 
     & = 2s(\gamma - \alpha s) + 2\rho\alpha\left(\frac{s^2} {\sqrt{s}\left|\gamma - \alpha s\right|}\right)(\gamma - \alpha s) \\
     & = 2s(\gamma - \alpha s)  + 2\text{sign}(\gamma - \alpha s)\rho\alpha s^{3/2}, 
\end{align*}
where by convention, $\text{sign}(x) = 0$ if $x = 0$, $= - 1$ if $x < 0$, and $ = 1$ if $x > 0$. 

Suppose the initializations are sufficiently small in the sense that we start with $\gamma - \alpha s > 0$, where the dynamics $\tau(\dot{s})_\text{SAM}$ is ``positive". Note that $s = \gamma/\alpha$ is a stable equilibrium. Therefore, $s$ will increase from initialization until it reaches this equilibrium and stays here. It never reaches the ``negative" part, so we can simplify to the following ODE for SAM: 
\[
    \tau \frac{ds}{dt} = 2s(\gamma - \alpha s)  + 2\rho\alpha s^{3/2}, \quad \text{where } \ \gamma - \alpha s > 0. 
\]
Lemma \ref{lem:ODE_closed_form_for_white_covariance} shows the closed-form solution for the readers' reference: 
\[
    \frac{1}{2\gamma}\ln\frac{s(t)}{s(0)} +\frac{1}{\alpha\Delta}\left( \frac{1}{r_-}\ln\frac{\sqrt{s(t)}-r_-}{\sqrt{s(0)}-r_-} -\frac{1}{r_+}\ln\frac{\sqrt{s(t)}-r_+}{\sqrt{s(0)}-r_+}\right) = \frac{t}{\tau}, \ \ r_\pm = \frac{\rho \pm \sqrt{\rho^2 + \frac{4\gamma}{\alpha}}}{2}.
\]
Some algebra verifies that $\rho = 0$ ($r_+ = \sqrt{\gamma/\alpha}, \ r_- = - r_+, \ \Delta = 2r_+$) yields: 
\begin{align*}
     & \frac{1}{2\gamma}\ln\frac{s(t)}{s(0)} +\frac{1}{2\alpha r_+}\left( -\frac{1}{r_+}\ln\frac{\sqrt{s(t)}+r_+}{\sqrt{s(0)}+r_+} -\frac{1}{r_+}\ln\frac{\sqrt{s(t)}-r_+}{\sqrt{s(0)}-r_+}\right)\\ 
     = \ &  \frac{1}{2\gamma}\ln\frac{s(t)}{s(0)} -\frac{1}{2\alpha r_+^2}\left(\ln\frac{s(t)-r_+^2}{s(0)-r_+^2}\right) \\ 
     = \ &  \frac{1}{2\gamma}\ln\frac{s(t)}{s(0)} -\frac{1}{2\gamma}\left(\ln\frac{s(t)-\gamma/\alpha}{s(0)-\gamma/\alpha}\right) \\ 
     = \ & \frac{1}{2\gamma}\ln\frac{s(t)(s(0)-\gamma/\alpha)}{s(0)(s(t)-\gamma/\alpha)}  = \frac{t}{\tau}. 
\end{align*}
From the last equation, we can solve: 
\[
    s(t) = \frac{(\gamma/\alpha)s(0)e^{2\gamma \frac{t}{\tau}}}{\gamma/\alpha + s(0)(e^{2\gamma \frac{t}{\tau}} - 1)} = \frac{\gamma e^{2\gamma \frac{t}{\tau}}}{\alpha(e^{2\gamma \frac{t}{\tau}} - 1) + \frac{\gamma}{s(0)}}. 
\]
This recovers the solution for GD dynamics in \cite{zhang2025training}, which has meaningful empirical interpretations. 
\end{proof}

Based on the ODE formula, we see that the SAM dynamics implements an additional boosting term $2\rho\alpha s^{3/2}$, which accelerates learning more strongly as $s$ grows (the relative speed-up scales like $\sqrt{s}$), leading to a faster approach to the GD trajectory’s operating region before saturation near $s=\gamma/\alpha$. 

The variance parameter $\alpha$ plays a dual role: it \emph{reduces} the steady-state level (since the GD carrying capacity is $\gamma/\alpha$) while simultaneously \emph{amplifying} the SAM gain. Intuitively, SAM is more effective at boosting learning when the features $\vz$ are more visible (larger variance $\alpha$).

While the implicit closed-form solution for SAM is less interpretable and thus not very interesting, we can simplify it given a small $s$. This yields an approximate ODE that \textbf{can extend to general covariances} and help quantify how long the model gets stuck around saddle points, as we show in the next section. 

\subsection{Escape Time from Saddles for General Covariance} \label{subsec:escape_time_for_general_covariance}

We assume the \textbf{same time parameterization} for SAM and GD just as a sanity check, and then we compute the time for the model to move from small initialization to a larger early weight. 

\begin{theorem}[Escape Time] \label{thm:escape_time}
    In the same setting as Theorem \ref{thm:merged_SAM_ODE}, except that we now allow general input covariance $\mSigma$, under vanishing initializations $s_0 = s(0)$, we denote the time it takes for SAM and GD to reach a reasonably larger early weight $s_*$ by $t_{\text{SAM}}$, $t_{\text{GD}}$ respectively. Then SAM yields slightly faster ``escape time": 
    \[
    t_{\text{GD}} = \frac{\tau}{\|\mSigma\|_F}\ln \left(\sqrt{\frac{s_*}{s_0}}\right)
    \] 
    \[
        t_{\text{SAM}} = \frac{\tau}{\|\mSigma\|_F}\ln \left(\frac{\sqrt{s_*}(\|\mSigma\|_F + \rho\alpha \sqrt{s_0})}{\sqrt{s_0}(\|\mSigma\|_F+ \rho\alpha \sqrt{s_*})}\right).
    \]
    If $\|\mSigma\|_F \gg \rho\alpha \sqrt{s_*}$, in particular, 
    \[
        \Delta t = t_{\text{GD}} - t_{\text{SAM}} = \frac{\tau\rho\alpha}{\|\mSigma\|_F^2}(\sqrt{s_*} - \sqrt{s_0}). 
    \]
\end{theorem}

\begin{proof} Recall the following two ODEs from Section \ref{subsec:merged_grad_hess}: 
\begin{enumerate}
    \item \textbf{GD: } $\tau \frac{ds}{dt} = 2s(\gamma - \alpha s)$.
    \item \textbf{SAM: } $\tau \frac{ds}{dt} = 2s(\gamma - \alpha s)  + 2\rho\alpha s^{3/2}$.
\end{enumerate}
Under sufficiently small initializations $s$ and the fact that the weights have not moved much during early training, we drop \emph{the least dominant term} and approximate the above ODEs as: 
\begin{enumerate} 
    \item \textbf{GD: } $\tau \frac{ds}{dt} = 2\gamma s$.
    \item \textbf{SAM: } $\tau \frac{ds}{dt} = 2\gamma s + 2\rho\alpha s^{3/2}$. 
\end{enumerate}

According to~\citep{zhang2025training, atanasov2021neural}'s perspectives, (1) we can also view this as keeping \emph{the most dominant term} in the gradient, which will yield the same ODE approximation via the new gradient: 
\begin{align*}
   & \frac{\partial \gL}{\partial \vw_2} =\nabla_{\vw_2}\gL = -2\mW_1\left(\mathbb{E}\left[y_q\vz^\top \right] - \underbrace{\vw_2^\top\mW_1}_{\text{two small weights}}\mathbb{E}\left[\vz\vz^\top\right]\right)^\top \approx  -2\mW_1\mathbb{E}\left[y_q\vz \right] + \text{small term} \\\implies 
   & \text{the approximation: }\ \ \ \tau\dot{\vw_2} =\mW_1\mathbb{E}\left[y_q\vz \right]. 
\end{align*}
(2) Additionally, despite a general nonwhite covariance, for early training where the loss has not changed steeply, the same equation in (\ref{eq:change_of_variable}) still holds
\[
    \mW_1 = \vw_2\frac{\mathbb{E}[y_q\vz^\top]}{\|\mathbb{E}[y_q\vz]\|_2}  \implies  \tau\dot{\vw_2} =\vw_2\frac{\mathbb{E}[y_q\vz^\top]}{\|\mathbb{E}[y_q\vz]\|_2}\mathbb{E}\left[y_q\vz \right] = \|\mathbb{E}[y_q\vz]\|_2\vw_2 = \gamma\vw_2. 
\]
Similarly, letting $s = \vw_2^\top\vw_2$ recovers exactly the ODE $\tau \frac{ds}{dt} = 2\gamma s$ for GD after dropping the least dominant term. And it follows that the corresponding SAM ODE is $\tau \frac{ds}{dt} = 2\gamma s + 2\rho\alpha s^{3/2}$. 

The approximate ODE for GD, same as in the original paper, is a basic separable equation, which has the following solution: 
\[
    s_{\text{GD}}(t) = s_{\text{GD}}(0)e^{2\gamma\frac{t}{\tau}}. 
\]
 By Lemma \ref{lem:ODE_closed_form_for_SAM_small_inits}, the SAM solution is: 
\[
    s_{\text{SAM}}(t) =\left(\frac{\gamma \sqrt{s_{\text{SAM}}(0)} e^{\gamma\frac{t}{\tau}}}{\gamma + \rho\alpha \sqrt{s_{\text{SAM}}(0)}\left(1 - e^{\gamma\frac{t}{\tau}}\right)}\right)^2. 
\]
Note that $\rho = 0$ recovers the GD solution. The above solutions provide analytical and tractable evolutions for these weights. For the early phase of training, we can then compute the time required to reach some $s(t)$ i.e escape from the saddle. 

Starting from the same small initializations where $s_0 :=s_{\text{GD}}(0) = s_{\text{SAM}}(0)$, we want to reach some larger early point $s_*$. We let $t_{\text{GD}}$, $t_{\text{SAM}}$ denote the time required such that $s_{\text{GD}}(t_{\text{GD}}) = s_*$ and $s_{\text{SAM}}(t_{\text{SAM}}) = s_*$ respectively. 

We can easily solve this by inverting the $s_{\text{GD}}$, $s_{\text{SAM}}$ functions above and have the following: 
\[
    t_{\text{GD}} = \frac{\tau}{\gamma}\ln \left(\sqrt{\frac{s_*}{s_0}}\right), 
\]
\[
    t_{\text{SAM}} = \frac{\tau}{\gamma}\ln \left(\frac{\sqrt{s_*}(\gamma + \rho\alpha \sqrt{s_0})}{\sqrt{s_0}(\gamma + \rho\alpha \sqrt{s_*})}\right). 
\]
Hence, we have: 
\[
    \Delta t = t_{\text{GD}} - t_{\text{SAM}}  = \frac{\tau}{\gamma}\ln \left(\frac{\gamma + \rho\alpha \sqrt{s_*}}{\gamma + \rho\alpha \sqrt{s_0}}\right) > 0 \quad \text{since } \ \ s_* > s_0. 
\]
If $\gamma \gg \rho\alpha\sqrt{s_*}$ (e.g. $\alpha = 1 + (1+d)/N$, $\gamma = \sqrt{d}$ when the input covariance $\mSigma = \mI$ by Equation 50 from \cite{zhang2025training}), we can apply $\ln(1+x) \approx x$ for small $x$ to obtain: 
\begin{align*}
    \Delta t & =\frac{\tau}{\gamma}\left(\ln\left(1 + \frac{\rho\alpha \sqrt{s_*}}{\gamma}\right)  - \ln\left(1 + \frac{\rho\alpha \sqrt{s_0}}{\gamma}\right)\right)   \approx \frac{\tau}{\gamma}\left(\frac{\rho\alpha}{\gamma}\right)(\sqrt{s_*} - \sqrt{s_0}). 
\end{align*}
Hence, SAM allows for a slightly faster convergence in terms of the gradient flow, and $\gamma = \|\mSigma\|_F^2 $ by~\citep{zhang2025training} Equation 53 completes the proof. 
\end{proof}

Notably, since the escape time is inversely proportional to the Frobenius norm of data covariance, increasing the norm via amplifying any data can accelerate convergence for both optimizers and continually reduce $\Delta t$. This is also apparent from Theorem \ref{thm:merged_SAM_ODE}, as increasing $\gamma= \|\mSigma\|_F^2$ will make the first term larger than the second term being subtracted and reduce its effect. In particular, this suggests that modifying data can be a viable path for GD to have SAM-like properties. However, \textbf{in practice, SAM should slow down convergence} in terms of both iterations and training time because the model spends time and steps correcting the adversarial moves. For vision models, \cite{nguyen2024changing} shows that this process enables uniform feature learning and reduces SB for better generalization. This contradiction could be due to that the model is too simple, and SAM often steps in a good direction instead of an adversarial one as intended. 

With this being said, we only provide the discussion of $\texttt{ATTN}_M$ here, as it does not show relevant insights about SB. 

\subsection{Helper Lemmas}

\begin{lemma} \label{lem:ODE_closed_form_for_white_covariance}
Consider the following ordinary differential equation: 
\[
    \tau\frac{ds}{dt} = 2s(\gamma -\alpha s) + 2\rho \alpha s^{3/2} \quad \text{where $\gamma-\alpha s>0$}. 
\]  
Then let $\Delta=\sqrt{\rho^2+\frac{4\gamma}{\alpha}}$, the solution satisfies: 
\[
    \frac{1}{2\gamma}\ln\frac{s(t)}{s(0)} +\frac{1}{\alpha\Delta}\left( \frac{1}{r_-}\ln\frac{\sqrt{s(t)}-r_-}{\sqrt{s(0)}-r_-} -\frac{1}{r_+}\ln\frac{\sqrt{s(t)}-r_+}{\sqrt{s(0)}-r_+}\right) = \frac{t}{\tau}, \ \ r_\pm = \frac{\rho \pm \Delta}{2}.
\]
\end{lemma}

\begin{proof}
Set $u(t)=\sqrt{s(t)}$. Skipping the argument $t$, we have $s=u^2$ and $\dot{s}=2u\dot{u}$. Substituting into the ODE gives: 
\begin{align*}
    \tau(2u\dot{u})&=2u^2(\gamma-\alpha u^2) + 2\rho\alpha u^3 \\
    \dot{u} &= \frac{u}{\tau}\Bigl(\gamma + \rho\alpha u - \alpha u^2\Bigr).
\end{align*}
This is separable:
\begin{equation} \label{eq:separated_SAM_ODE1}
    \frac{du}{u(\gamma+\rho\alpha u - \alpha u^2)} = \frac{dt}{\tau}.
\end{equation}
We perform partial fraction decomposition on the LHS. First, 
\[
    \gamma+\rho\alpha u-\alpha u^2 = -\alpha(u-r_+)(u-r_-), \quad r_\pm=\frac{\rho \pm \sqrt{\rho^2+4\gamma/\alpha}}{2}.
\]
We want to solve for $A$, $B$, $C$ in the following equation: 
\begin{align*}
    \frac{1}{-\alpha u(u-r_+)(u-r_-)}  =\frac{A}{u}+\frac{B}{u-r_+}+\frac{C}{u-r_-}
\end{align*}
\[
    \implies 1=-\alpha\left[A(u-r_+)(u-r_-)+Bu(u-r_-)+Cu(u-r_+)\right]. 
\]
Let $u = 0$. This equation simplifies to: 
\[
    1=-\alpha A(-r_+)(-r_-)\quad\Longrightarrow\quad A=\frac{1}{-\alpha r_+r_-} = \frac{1}{\gamma}.
\]
Similarly, with $u=r_+$ and $u=r_-$ respectively, we have: 
\[
    1=-\alpha B\,r_+(r_+-r_-)\quad\Longrightarrow\quad B=-\frac{1}{\alpha r_+(r_+-r_-)}.
\]
\[
    1=-\alpha Cr_-(r_- - r_+) =\alpha Cr_-(r_+-r_-) \quad\Longrightarrow\quad C=\frac{1}{\alpha r_-(r_+-r_-)}.
\]
With the partial fraction decomposition, we integrate both sides of Equation \ref{eq:separated_SAM_ODE1}: 
\begin{align*}
    \int \frac{du}{u(\gamma+\rho\alpha u - \alpha u^2)} & = A\ln u +  B\ln|u-r_+| + C\ln|u-r_-|\\ 
    & = \frac{1}{\gamma}\ln u -\frac{1}{\alpha r_+(r_+-r_-)}\ln|u-r_+| + \frac{1}{\alpha r_-(r_+-r_-)}\ln|u-r_-| \\ 
    & = \frac{1}{\gamma}\ln u + \frac{1}{\alpha \Delta}\left(\frac{1}{r_-}\ln|u-r_-| - \frac{1}{r_+}\ln|u-r_+|\right) \\ 
    & = \frac{t}{\tau} + C, 
\end{align*}
where $\Delta = r_+ - r_- = \sqrt{\rho^2 + \frac{4\gamma}{\alpha}}$. 
Imposing the initial condition $u(0)$ gives
\[
\frac{1}{\gamma}\ln\!\frac{u(t)}{u_0}
+\frac{1}{\alpha\Delta}\!\left(
\frac{1}{r_-}\ln\!\frac{u(t)-r_-}{u_0-r_-}
-\frac{1}{r_+}\ln\!\frac{u(t)-r_+}{u_0-r_+}
\right)
= \frac{t}{\tau}.
\]
Finally, substituting back $u(t) = \sqrt{s(t)}$ completes the proof.
\end{proof}

\begin{lemma} \label{lem:ODE_closed_form_for_SAM_small_inits}
Consider the following ordinary differential equation: 
\[
    \tau \frac{ds}{dt} = 2\gamma s + 2\rho\alpha s^{3/2}, \quad \text{where $\gamma-\alpha s>0$}. 
\]  
It has the following solution: 
\[
    s(t) =\left(\frac{\gamma \sqrt{s(0)} e^{\gamma\frac{t}{\tau}}}{\gamma + \rho\alpha \sqrt{s(0)}\left(1 - e^{\gamma\frac{t}{\tau}}\right)}\right)^2. 
\]
\end{lemma}

\begin{proof}
Set $u(t)=\sqrt{s(t)}$. We have $s=u^2$ and $\dot{s}=2u\dot{u}$. Substituting into the ODE gives: 
\begin{align*}
    \tau(2u\dot{u})&=2\gamma u^2 + 2\rho\alpha u^3  \quad \Longrightarrow \quad \dot{u} = \frac{u}{\tau}\left(\gamma+ \rho\alpha u\right).
\end{align*}
This equation is separable:
\begin{equation} \label{eq:separated_SAM_ODE2}
    \frac{du}{u(\gamma+ \rho\alpha u)} = \frac{dt}{\tau}.
\end{equation}
We perform partial fraction decomposition on the LHS and obtain, 
\[
    \int\frac{1}{u(\gamma+ \rho\alpha u)} du = \int \frac{1}{\gamma}\frac{1}{u} - \frac{\rho\alpha}{\gamma} \frac{1}{\gamma+ \rho\alpha u} du \implies
    \frac{1}{\gamma} \ln u - \frac{1}{\gamma} \ln (\gamma+ \rho\alpha u) = \frac{1}{\gamma} \ln \left(\frac{u}{\gamma+ \rho\alpha u}\right) = \frac{t}{\tau} + C. 
\]
Imposing the initial condition $u(0)$ gives
\[
\ln \left(\frac{u}{\gamma+ \rho\alpha u}\right) -  \ln \left(\frac{u(0)}{\gamma+ \rho\alpha u(0)}\right)= \gamma\frac{t}{\tau} \implies
    u(t) = \left(\frac{\gamma u(0) e^{\gamma\frac{t}{\tau}}}{\gamma + \rho\alpha u(0)\left(1 - e^{\gamma\frac{t}{\tau}}\right)}\right)^2. 
\]
Finally, substituting back $u(t) = \sqrt{s(t)}$ completes the proof.
\end{proof}

    \section{Proofs of Theorems for $\texttt{ATTN}_S$}  
\label{app:proof_for_separate}

We now focus on proving Theorems \ref{thm:SAM_ODE} and \ref{thm:SAM_reduces_SB} in the main text, showing that SAM reduces SB early in training.  

\subsection{Preliminary and Our Hessian Extension}

We first recall and summarize some important properties and formulas from~\citep{zhang2025training} (with slightly different notations), in particular from Section E and Subsections C.1, C.2. Given the model $\texttt{ATTN}_S$, we have the following gradients for the $i$-th head: 
\begin{align}
    \tau \dot{v}_i & = \vk_i^{\top} \mathbb{E}\left[ \vbeta \left(y_q - \hat{y}_q\right)\vx_q^{\top} \right] \vq_i = \vk_i^{\top} \left( \mSigma^{2} - \mathbb{E}\left[ \hat{\mSigma}^{2} \right] 
   \sum_{i'=1}^{H} v_{i'} \vk_{i'} \vq_{i'}^{\top} \mSigma \right) \vq_i, \label{eq:tau_v_GD} \\
    \tau \dot{\vk}_i &= v_i \mathbb{E}\left[ \vbeta \left(y_q - \hat{y}_q\right)\vx_q^{\top} \right] \vq_i =  v_i \left( \mSigma^{2} - \mathbb{E}\left[ \hat{\mSigma}^{2} \right] 
   \sum_{i'=1}^{H} v_{i'} \vk_{i'} \vq_{i'}^{\top} \mSigma \right) \vq_i, \label{eq:tau_k_GD} \\
    \tau \dot{\vq}_i &= v_i \mathbb{E}\left[ \vx_q \left(y_q - \hat{y}_q\right)\vbeta^{\top} \right] \vk_i = v_i \left( \mSigma^{2} - \mSigma \sum_{i'=1}^{H} v_{i'} \vk_{i'} \vq_{i'}^{\top} 
   \mathbb{E} \left[ \hat{\mSigma}^{2} \right] \right) \vk_i \label{eq:tau_q_GD}.
\end{align}
Equivalently, for the squared loss, we have the following gradient: 
\begin{align} 
    \frac{\partial \gL}{\partial v_i} & = \nabla_{v_i}\gL 
    = -2 \vk_i^{\top} \left( \mSigma^{2} - \mathbb{E}\left[ \hat{\mSigma}^{2} \right] 
    \sum_{i'=1}^{H} v_{i'} \vk_{i'} \vq_{i'}^{\top} \mSigma \right) \vq_i, \label{eq:v_gradients}\\
    \frac{\partial \gL}{\partial \vk_i} & = \nabla_{\vk_i}\gL
    = -2 v_i \left( \mSigma^{2} - \mathbb{E}\left[ \hat{\mSigma}^{2} \right]
    \sum_{i'=1}^{H} v_{i'} \vk_{i'} \vq_{i'}^{\top} \mSigma \right) \vq_i, \label{eq:k_gradients} \\
    \frac{\partial \gL}{\partial \vq_i} & = \nabla_{\vq_i}\gL
    = -2 v_i \left( \mSigma^{2} - \mSigma \sum_{i'=1}^{H} v_{i'} \vk_{i'} \vq_{i'}^{\top} 
    \mathbb{E}\left[ \hat{\mSigma}^{2} \right] \right) \vk_i \label{eq:q_gradients}.
\end{align}
Here $\mSigma = \sum_{i=1}^d \lambda_i\ve_i\ve_i^\top$ and $\hat{\mSigma }$ denote the population covariance and sample covariance matrices, respectively. In particular, their Equations 30 and 31 show that: 
\[
    \mathbb{E}\left[ \hat{\mSigma}^2 \right] 
    = \mSigma^2 + \frac{\mSigma + \Tr(\mSigma) \mI}{N}\mSigma = \sum_{i=1}^d a_i \ve_i \ve_i^\top, 
\]
where the eigenvectors are the same as $\mSigma$ and 
\[
a_i = \left(1 + \frac{1}{N}\right)\lambda_i^2 
+ \frac{\Tr(\mSigma)}{N}\lambda_i = \lambda_i^2 \left( 1 + \frac{1 + \Tr(\mSigma)/\lambda_i}{N} \right).
\]

On top of the gradient computations, we perform the following Hessian computation for the $i$-th head weight $\vw_i = \left[v_i \ \ \vk_i^\top \ \ \vq_i^\top\right]^\top \in \mathbb{R}^{2d+1}$  (``with itself''). The Hessian has the $3 \times 3$ block structure
\[
    \mH_i = 
    \begin{bmatrix}
        \mH_{i, vv} & \mH_{i, vk} & \mH_{i, vq} \\ 
        \mH_{i, kv} & \mH_{i, kk} & \mH_{i, kq} \\
        \mH_{i, qv} & \mH_{i, qk} & \mH_{i, qq}
    \end{bmatrix},  \quad \text{where}
\]
\begin{align*}
    \mH_{i, vv} &= \frac{\partial}{\partial v_i}\left(\nabla_{v_i}\gL \right)
    = 2 \vk_i^{\top}  \mathbb{E}\left[ \hat{\mSigma}^{2} \right] \vk_{i} \vq_{i}^{\top} \mSigma  \vq_i, \\
    \mH_{i, kv} &= \mH_{i, vk}^\top = \frac{\partial}{\partial \vk_i}\left(\nabla_{v_i}\gL \right)
    = -2 \mSigma^{2}\vq_i + 2\mathbb{E}\left[ \hat{\mSigma}^{2} \right] 
    \sum_{i'=1}^{H} v_{i'} \vk_{i'} \vq_{i'}^{\top} \mSigma \vq_i  + 2v_i\vq_{i}^{\top} \mSigma \vq_i \mathbb{E}\left[ \hat{\mSigma}^{2} \right] 
    \vk_{i}, \\
    \mH_{i, kk} &= \frac{\partial}{\partial \vk_i^\top}\left(\nabla_{\vk_i}\gL \right) = 2 v_{i}^2 \vq_{i}^{\top} \mSigma \vq_i\mathbb{E}\left[ \hat{\mSigma}^{2} \right], \\
    \mH_{i, kq} &= \mH_{i, qk}^\top = \frac{\partial}{\partial \vq_i^\top}\left(\nabla_{\vk_i}\gL \right)
    = -2  v_i \mSigma^{2} + 2 v_i \mathbb{E}\left[ \hat{\mSigma}^{2} \right]
    \sum_{i'=1}^{H} v_{i'} \vk_{i'} \vq_{i'}^{\top} \mSigma + 2 v_i^2 \mathbb{E}\left[ \hat{\mSigma}^{2} \right] \vk_{i}  \vq_i^\top \mSigma, \\
    \mH_{i, qq} &= \frac{\partial}{\partial \vq_i^\top}\left(\nabla_{\vq_i}\gL \right) = 2 v_i^2 \mSigma \vk_{i} \vk_i^{\top} 
    \mathbb{E}\left[ \hat{\mSigma}^{2} \right] , \\
    \mH_{i, qv} &= \mH_{i, vq}^\top = \frac{\partial}{\partial \vq_i}\left(\nabla_{v_i}\gL \right)
    = -2 \mSigma^{2} \vk_i + 2\mSigma 
    \sum_{i'=1}^{H} v_{i'} \vq_{i'} \vk_{i'}^{\top} \mathbb{E}\left[ \hat{\mSigma}^{2} \right]\vk_i + 2v_{i}\vk_i^\top \mathbb{E}\left[ \hat{\mSigma}^{2} \right] 
    \vk_{i} \mSigma \vq_i. 
\end{align*}

In the above Hessian, we do not need to include other heads, as the heads become active for learning one at a time, and the others do not induce new gradient in SAM.  

\subsection{Learning the First Dominant Eigenvector: Proof of Theorem \ref{thm:SAM_ODE} Part 1}

We permute and index the heads by the order of learning the eigenvectors. At the start of training, the model learns the most prominent eigenvector $\ve_1$, and recall that by the ansatz, the weights are set as: 
\[
    \vk_1 = \vq_1 = v_1\ve_1; \quad \vk_i = \vq_i = \bm{0}, \ v_i = 0 \quad \text{for other heads. }
\]
The head indexed $1$ first becomes active, and the other parallel heads are not affected during the training (zero dynamics). 

We first have the following simplified GD dynamics from Equations \ref{eq:tau_v_GD}, \ref{eq:tau_k_GD}, \ref{eq:tau_q_GD}: 
\begin{align*}
&\mSigma^{2} - \mathbb{E}\left[\hat{\mSigma}^{2} \right] 
   \sum_{i=1}^{H} v_i \vk_i \vq_i^{\top} \mSigma = \mSigma^{2} - \sum_{i=1}^{d} a_i \ve_i \ve_i^{\top} v_1^{3} \ve_1 \ve_1^{\top} \mSigma = \mSigma^{2} - \lambda_1 a_1 \ve_1 \ve_1^{\top} v_1^{3} \\
&\tau \dot{v}_1 = v_1^{2} \ve_1^{\top} \left( \mSigma^{2} - \lambda_1 a_1 \ve_1 \ve_1^{\top} v_1^{3} \right) \ve_1
   = \lambda_1^{2} v_1^{2} - \lambda_1 a_1 v_1^{5}, \quad 
\tau \dot{\vk}_1 = \lambda_1^{2} v_1^{2} \ve_1 - \lambda_1 a_1 v_1^{5} \ve_1, \\
& \tau \dot{\vq}_1 = \lambda_1^{2} v_1^{2} \ve_1 - \lambda_1 a_1 v_1^{5} \ve_1, \quad v_i = 0, \  \vk_i = 0,\  \vq_i = 0,\quad i = 2, \ldots, H .
\end{align*}
This significantly reduces the complexity of the dynamics. Note that we simply multiply by $-2$ to obtain the gradients. Now we want to introduce SAM and extend the computation to the Hessian: 
\begin{align*}
    \mH_{1, vv} &= 2 \vk_1^{\top}  \mathbb{E}\left[ \hat{\mSigma}^{2} \right] \vk_{1} \vq_{1}^{\top} \mSigma  \vq_1 = 2v_1^4\ve_1^\top \mathbb{E}\left[ \hat{\mSigma}^{2} \right] \ve_1 \ve_1^{\top} \mSigma  \ve_1 = 2\lambda_1 a_1 v_1^4 . 
\end{align*}

Similarly, we can compute
\begin{align*}
    \mH_{1, kv} &= \mH_{1, vk}^\top = -2\lambda_1^2v_1 \ve_1 + 2v_1^4\mathbb{E}\left[ \hat{\mSigma}^{2} \right] 
     \ve_1 \ve_1^\top \mSigma \ve_1  + 2v_1^4\ve_1^{\top} \mSigma \ve_1 \mathbb{E}\left[ \hat{\mSigma}^{2} \right] \ve_1 = (-2\lambda_1^2v_1 + 4\lambda_1 a_1 v_1^4)\ve_1, \\
    \mH_{1, kk} & = 2 v_{1}^4 \ve_{1}^{\top} \mSigma \ve_{1}\mathbb{E}\left[ \hat{\mSigma}^{2} \right] = 2\lambda_1 v_1^4\mathbb{E}\left[ \hat{\mSigma}^{2} \right], \\
    \mH_{1, kq} &= -2  v_1 \mSigma^{2} + 2 v_1^4 \mathbb{E}\left[ \hat{\mSigma}^{2} \right] \ve_{1} \ve_{1}^{\top} \mSigma + 2 v_1^4 \mathbb{E}\left[ \hat{\mSigma}^{2} \right] \ve_{1}  \ve_1^\top \mSigma = -2  v_1 \mSigma^{2} + 4\lambda_1 a_1 v_1^4 \ve_1 \ve_1^\top, \\
    \mH_{1, qq} &=  2 v_1^4 \mSigma \ve_1 \ve_1^{\top} 
    \mathbb{E}\left[ \hat{\mSigma}^{2} \right] = 2\lambda_1 a_1 v_1^4 \ve_1 \ve_1^\top, \\
    \mH_{1, qv} &=  -2 v_1 \mSigma^{2} \ve_1 + 2v_1^4 \mSigma 
     \ve_{1} \ve_{1}^{\top} \mathbb{E}\left[ \hat{\mSigma}^{2} \right]\ve_1 + 2v_{1}^4 \ve_1^\top \mathbb{E}\left[ \hat{\mSigma}^{2} \right] 
    \ve_{1} \mSigma \ve_1 = (-2 \lambda_1^2 v_1 + 4\lambda_1 a_1 v_1^4) \ve_1. 
\end{align*}

Recall that the SAM and GD dynamics can be related via the following first-order approximation: 
\begin{equation} \label{eq:SAM_separate_P}
    \tau (\dot{\vw}_1)_{\text{SAM}} = \begin{bmatrix}
         \tau \dot{v}_1 \\  \tau \dot{\vk}_1\\  \tau \dot{\vq}_1
    \end{bmatrix}_{\text{SAM}} = \left(\mI_{2d+2} + \mP \right)\begin{bmatrix}
         \tau \dot{v}_1 \\  \tau \dot{\vk}_1\\  \tau \dot{\vq}_1
    \end{bmatrix}_{\text{GD}}, \quad \mP = \frac{\rho}{\|\nabla_{\vw_1}\gL\|_2}\mH_1. 
\end{equation}
We plug in the relevant quantities: 
\begin{align*}
    \left(\mI_{2d+2} + \mP \right)\begin{bmatrix}
         \tau \dot{v}_1 \\  \tau \dot{\vk}_1\\  \tau \dot{\vq}_1
    \end{bmatrix}_{\text{GD}} = \underbrace{\begin{bmatrix}
         \lambda_1^{2} v_1^{2} - \lambda_1 a_1 v_1^{5} \\  (\lambda_1^{2} v_1^{2} - \lambda_1 a_1 v_1^{5})\ve_1\\  (\lambda_1^{2} v_1^{2} - \lambda_1 a_1 v_1^{5})\ve_1  
    \end{bmatrix}}_{\text{GD dynamics}} + \underbrace{\mP\begin{bmatrix}
         \lambda_1^{2} v_1^{2} - \lambda_1 a_1 v_1^{5} \\  (\lambda_1^{2} v_1^{2} - \lambda_1 a_1 v_1^{5})\ve_1\\  (\lambda_1^{2} v_1^{2} - \lambda_1 a_1 v_1^{5})\ve_1  
    \end{bmatrix}}_{\text{SAM correction $=: \mQ$}}
\end{align*}
It remains to compute the perturbation matrix. First note $\|\nabla_{\vw_1}\gL\|_2 = \sqrt{\|\nabla_{v_1}\gL\|_2 + \|\nabla_{\vk_1}\gL\|_2 + \|\nabla_{\vq_1}\gL\|_2} = 2\sqrt{3}\lvert \lambda_1^2 v_1^2 - \lambda_1 a_1 v_1^5\rvert$ by Equations \ref{eq:v_gradients}, \ref{eq:k_gradients}, \ref{eq:q_gradients}. Furthermore, due to sufficiently small initializations, $\lvert \lambda_1^2 v_1^2 - \lambda_1 a_1 v_1^5\rvert =  \lambda_1^2 v_1^2 - \lambda_1 a_1 v_1^5$ and the SAM correction term, according to \ref{eq:SAM_separate_P}, becomes: 
\begin{align*}
    \mQ & = \frac{\rho}{2\sqrt{3}}\begin{bmatrix}
        2\lambda_1 a_1 v_1^4 & (-2\lambda_1^2v_1 + 4\lambda_1 a_1 v_1^4)\ve_1^\top & (-2 \lambda_1^2 v_1 + 4\lambda_1 a_1 v_1^4) \ve_1^\top \\
        (-2\lambda_1^2v_1 + 4\lambda_1 a_1 v_1^4)\ve_1 & 2\lambda_1 v_1^4\mathbb{E}[ \hat{\mSigma}^{2} ] & -2  v_1 \mSigma^{2} + 4\lambda_1 a_1 v_1^4 \ve_1 \ve_1^\top \\ 
        (-2 \lambda_1^2 v_1 + 4\lambda_1 a_1 v_1^4) \ve_1 & -2  v_1 \mSigma^{2} + 4\lambda_1 a_1 v_1^4 \ve_1 \ve_1^\top & 2\lambda_1 a_1 v_1^4 \ve_1 \ve_1^\top
    \end{bmatrix} \begin{bmatrix}
         1 \\ \ve_1\\  \ve_1  
    \end{bmatrix} \\ 
    & = \frac{\rho}{2\sqrt{3}}\begin{bmatrix}
        2\lambda_1 a_1 v_1^4 + (-2\lambda_1^2v_1 + 4\lambda_1 a_1 v_1^4)+ (-2 \lambda_1^2 v_1 + 4\lambda_1 a_1 v_1^4) \\
        (-2\lambda_1^2v_1 + 4\lambda_1 a_1 v_1^4)\ve_1 + 2\lambda_1 a_1 v_1^4 \ve_1  -2  \lambda_1^2v_1 \ve_1 + 4\lambda_1 a_1 v_1^4 \ve_1  \\ 
        (-2 \lambda_1^2 v_1 + 4\lambda_1 a_1 v_1^4) \ve_1  -2  \lambda_1^2 v_1 \ve_1 + 4\lambda_1 a_1 v_1^4 \ve_1 + 2\lambda_1 a_1 v_1^4 \ve_1 
    \end{bmatrix} \\ 
    & = \frac{\rho}{2\sqrt{3}}\begin{bmatrix}
        10\lambda_1 a_1 v_1^4 - 4\lambda_1^2v_1  \\
        (10\lambda_1 a_1 v_1^4 - 4\lambda_1^2v_1)\ve_1  \\ 
        (10\lambda_1 a_1 v_1^4 - 4\lambda_1^2v_1) \ve_1 
    \end{bmatrix}. 
\end{align*}
Finally, here we drop the subscript and have the following SAM dynamics: 
\begin{align*}
     \tau\dot{v}_1 &= \lambda_1^{2} v_1^{2} - \lambda_1 a_1 v_1^{5} + \frac{\rho}{\sqrt{3}}\left(5\lambda_1 a_1 v_1^4 - 2\lambda_1^2v_1\right), \\
     \tau\dot{\vk}_1 &= \left(\lambda_1^{2} v_1^{2} - \lambda_1 a_1 v_1^{5} + \frac{\rho}{\sqrt{3}}\left(5\lambda_1 a_1 v_1^4 - 2\lambda_1^2v_1\right) \right)\ve_1, \\ 
     \tau\dot{\vq}_1 &= \left(\lambda_1^{2} v_1^{2} - \lambda_1 a_1 v_1^{5} + \frac{\rho}{\sqrt{3}}\left(5\lambda_1 a_1 v_1^4 - 2\lambda_1^2v_1\right) \right)\ve_1. 
\end{align*}
Similar to Section \ref{subsec:escape_time_for_general_covariance}, under small initializations, we can drop the lower order terms and obtain the following ODE's for SAM and GD: 

\textbf{GD Dynamics: }
\begin{align*}
     \tau\dot{v}_1 = \lambda_1^{2} v_1^{2}, \quad \quad \tau\dot{\vk}_1 = \lambda_1^{2} v_1^{2} \ve_1, \quad \quad \tau\dot{\vq}_1 = \lambda_1^{2} v_1^{2} \ve_1. 
\end{align*}
\textbf{SAM Dynamics: }
\begin{align*}
     \tau\dot{v}_1 = \lambda_1^{2} v_1^{2}  - \frac{2\rho}{\sqrt{3}}\lambda_1^2v_1, \quad \quad \tau\dot{\vk}_1 = \left( \lambda_1^{2} v_1^{2}  - \frac{2\rho}{\sqrt{3}}\lambda_1^2v_1 \right)\ve_1, \quad\quad   \tau\dot{\vq}_1 = \left( \lambda_1^{2} v_1^{2}  - \frac{2\rho}{\sqrt{3}}\lambda_1^2v_1\right)\ve_1. 
\end{align*}
We note that we can reduce them to one scalar-variable ODE: 
\begin{enumerate}
    \item \textbf{GD}: \quad $\tau\dot{v}_1 = \lambda_1^{2} v_1^{2}$,
    \item \textbf{SAM}: \quad $\tau\dot{v}_1 = \lambda_1^{2} v_1^{2}  - \frac{2\rho}{\sqrt{3}}\lambda_1^2v_1 $.
\end{enumerate}

\subsection{Saddle-to-Saddle Learning: Proof of Theorem \ref{thm:SAM_ODE} Part 2}

It turns out that when the model is at the $(m+1)$-th plateau for $m \geq 1$, namely the model is learning the $(m+1)$-th feature, the learning dynamics can be captured by the same ODEs. Recall from the ansatz that:
\[
    \vk_i = \vq_i = v_i\ve_i = \left(\lambda_i + \frac{\lambda_i + \Tr(\mSigma)}{N}\right)^{-1/3} \ve_i, \ 1 \leq i \leq m, \quad \vk_{m+1} = \vq_{m+1} = v_{m+1}(t)\ve_{m+1}
\]
\[
    \vk_i = \vq_i =  \bm{0},  v_i = 0, \ m+2 \leq i \leq H \text{ (the other heads)}
\]
We first have the following simplified GD dynamics for $v_{m+1}$ from Equation \ref{eq:tau_v_GD}: 
\begin{align*}
\tau \dot{v}_{m+1} &  = \vk_{m+1}^{\top} \left( \mSigma^{2} - \mathbb{E}\left[ \hat{\mSigma}^{2} \right] 
   \sum_{i'=1}^{H} v_{i'} \vk_{i'} \vq_{i'}^{\top} \mSigma \right) \vq_{m+1} \\ 
   & = \lambda_{m+1}^2v_{m+1}^2 - v_{m+1}^2\ve_{m+1}^\top\sum_{i=1}^{d} a_i \ve_i \ve_i^{\top}\sum_{i'=1}^{H} v_{i'}^3 \ve_{i'} \ve_{i'}^{\top}\ve_{m+1} \\ 
   & = \lambda_{m+1}^2v_{m+1}^2 -\lambda_{m+1} a_{m+1} v_{m+1}^{5}, 
\end{align*}
where the last step follows from orthonormal eigenvectors. This is the same dynamics as $v_1$ (just a different index), and the same applies for $\vk_{m+1}$ and $\vq_{m+1}$. 

Calculating the SAM perturbation $\mP$ further requires the Hessian $\mH_{m+1}$, which follows the same process as computing $\mH_1$. Hence, the perturbation part is also the same up to replacing the index.

\subsection{Learning Entropy: Proof of Theorem \ref{thm:SAM_reduces_SB}} \label{subsec:proof_of_SAM_reduces_SB}

We first refer to Fig. \ref{fig:toy_entropy_example} to provide a toy instantiation of how the entropy of the normalized time sequence captures uniformity in feature learning. Note that after sum-normalization, the discrete time sequence is converted into a probability mass distribution, and more uniform learning (that is more even time spent on the features) naturally becomes ``close" to a uniform distribution. We quantify the notion of closeness via Shannon entropy, for which the uniform distribution is the largest. 

\begin{figure*}[t]
    \centering
    \includegraphics[width=0.9\textwidth]{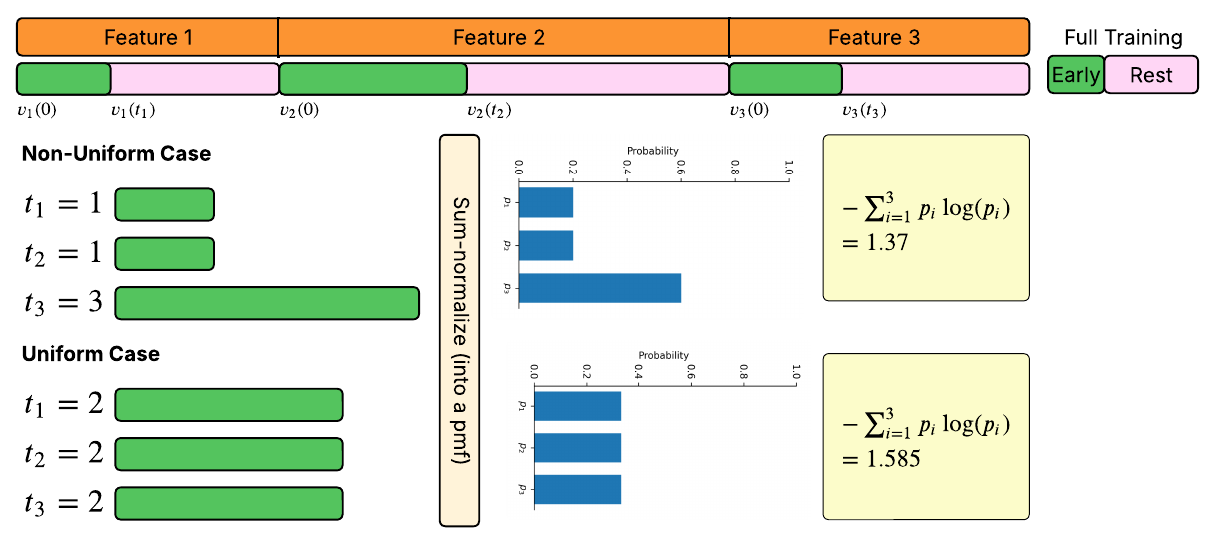}
    \caption{Toy example of Shannon entropy in quantifying uniformity in feature learning. We consider the extreme case of only 3 features. Following the $\texttt{ATTN}_S$ setting, the training can be separated by features. For each feature, there exists an early training phase (in green), and $t_1, t_2, t_3$ mark the times spent in this regime. More uniform learning is achieved when these values are closer to each other (bottom case), and after sum-normalization, this naturally corresponds to a distribution that resembles the uniform distribution (e.g. $(1/3, 1/3, 1/3)$ over three objects) and hence has higher entropy}
    
    \label{fig:toy_entropy_example}
\end{figure*}

Now consider the following two functions: 
\[
    I_1(u) := \int_{\varepsilon}^{u}\frac{1}{v^2} \ dv, \quad I_2(u) := \int_{\varepsilon}^{u}\frac{1}{v^2 - \hat{\rho}v} \ dv, \quad \hat{\rho} = \frac{2\rho}{\sqrt{3}}.
\]
Then letting $u_i = c\lambda_i^{-1/3}$, we have: 
\begin{equation} \label{eq:GD_SAM_time_integral}
     t_i^{\text{GD}} = \frac{\tau}{\lambda_i^2}I_1(u_i) = \tau \left(\frac{u_i}{c}\right)^6I_1(u_i),  \quad t_i^{\text{SAM}} = \frac{\tau}{\lambda_i^2}I_2(u_i) = \tau \left(\frac{u_i}{c}\right)^6I_2(u_i). 
\end{equation}
\textbf{Step 1:} Note that $t_i^{\text{GD}}$ and $t_i^{\text{SAM}}$ are both strictly increasing in $u_i$. This is straightforward as $I_1(u)$ and $I_2(u)$ are both increasing in $u$ (when $\varepsilon > \hat{\rho}$).

\textbf{Step 2:} Show that $t_i^{\text{SAM}}$ is a decreasing ``reweighting" of $t_i^{\text{GD}}$. More precisely, we start by defining the ratio: 
\[
    w(u) := \frac{I_2(u)}{I_1(u)} = \frac{t^{\text{SAM}}(u)}{t^{\text{GD}}(u)} \quad \text{if we treat them as continuous.}
\]
We first note that: 
\[
    w(u) = \frac{\int_{\varepsilon}^{u}\frac{1}{v^2 - \hat{\rho}v} \ dv}{\int_{\varepsilon}^{u}\frac{1}{v^2} \ dv} = \frac{\int_{\varepsilon}^{u}\phi(v)\frac{1}{v^2} \ dv}{\int_{\varepsilon}^{u}\frac{1}{v^2} \ dv}, \quad \text{where} \ \phi(v) = \frac{v}{v- \hat{\rho}}. 
\]
Note that $\phi(v)$ is strictly decreasing on $[\varepsilon, \infty)$ i.e. $\phi^\prime(v) = \frac{-\hat{\rho}}{(v-\hat{\rho})^2} < 0$. This implies that: 
\begin{equation} \label{eq:w_vs_phi_bound}
    w(u) = \frac{I_2(u)}{I_1(u)}> \phi(u)\frac{\int_{\varepsilon}^{u}\frac{1}{v^2} \ dv}{\int_{\varepsilon}^{u}\frac{1}{v^2} \ dv} = \phi(u) = \frac{u}{u - \hat{\rho}} = \frac{I_2^\prime(u)}{I_1^\prime(u)}. 
\end{equation}
We now differentiate and obtain: 
\begin{align*}
    \text{sign}(w^\prime(u)) = \text{sign}\left(\frac{I_1(u)I_2^\prime(u) - I_1^\prime(u)I_2(u)}{I_1(u)^2}\right) & =\text{sign}\left(I_1(u)I_2^\prime(u) - I_1^\prime(u)I_2(u)\right) \\
    & =\text{sign}\left(\frac{I_2^\prime(u) }{I_1^\prime(u) }I_1(u)- I_2(u)\right)  < 0, 
\end{align*}
where the last inequality directly follows from Equation \ref{eq:w_vs_phi_bound}. This shows the claim. 

\textbf{Step 3: } It follows from Lemma \ref{lem:incweight-majorized} when we let: 
\begin{align*}
    &\text{The non-increasing sequence } A = \text{the sequence } \frac{t_i^{\text{GD}}}{\sum_jt_j^{\text{GD}}} \  (i = 1, \dots, M) \ \text{ reversed (decreasing by Step 1)}; \\ 
    & \text{The non-increasing sequence } B = \text{the sequence } \frac{t_i^{\text{SAM}}}{\sum_jt_j^{\text{SAM}}} \  (i = 1, \dots, M) \ \text{ reversed (decreasing by Step 1)}; \\
    & \text{The non-decreasing sequence } w = \text{the sequence } \frac{t_i^{\text{SAM}}}{t_i^{\text{GD}}} \  (i = 1, \dots, M) \ \text{ reversed (increasing by Step 2)}, 
\end{align*}
we have that the normalized GD time sequence majorizes the normalized SAM time sequence. By the strict Schur-concavity of Shannon entropy (Lemma \ref{lem:entropy_is_Schur_concave}), 
    \[
         \ent\left(\left\{t_i^{\text{SAM}}\right\}_{i=1}^M\right) >  \ent\left(\left\{t_i^{\text{GD}}\right\}_{i=1}^M\right).
    \]
This completes the proof of reduced simplicit bias. 

\subsection{Majorization Theory: Helper Lemmas}

In mathematics, majorization theory is a natural branch for studying spread because it gives a principled way to say when one set of numbers is more uneven (more “peaked”) than another—without committing to a single statistic like variance, which can be unreliable at times. We provide the 

\begin{definition} (Restatement of Vector Majorization, Definition A.1, Chapter 1\cite{marshall1979inequalities}) \label{def:majorization} Let $x,y\in\mathbb{R}^n$. Let $x^\downarrow$ and $y^\downarrow$ denote the vectors obtained by sorting the coordinates of $x$ and $y$ in \emph{nonincreasing} order:
\[
x_1^\downarrow \ge x_2^\downarrow \ge \cdots \ge x_n^\downarrow,
\qquad
y_1^\downarrow \ge y_2^\downarrow \ge \cdots \ge y_n^\downarrow.
\]
We say that $x$ \emph{majorizes} $y$, and write $x \succ y$, if
\[
\sum_{i=1}^k x_i^\downarrow \;\ge\; \sum_{i=1}^k y_i^\downarrow
\quad \text{for all } k=1,2,\dots,n-1, \quad \text{ and } \ \ \sum_{i=1}^n x_i^\downarrow \;=\; \sum_{i=1}^n y_i^\downarrow.
\]
    
\end{definition}

\begin{lemma}[Increasing-weight renormalization is majorized]\label{lem:incweight-majorized}
Let \(A=(A_1,\dots,A_n)\) be a probability vector with
\[
A_1 \geq A_2 \geq \dots \geq A_n > 0,\quad \sum_{i=1}^n A_i = 1.
\]
Let \(w=(w_1,\dots,w_n)\) be a nondecreasing sequence of positive weights $
0 < w_1 \le w_2 \le \cdots \le w_n,$
and define the renormalized reweighting \(B=(B_1,\dots,B_n)\) by
\[
B_i := \frac{A_i w_i}{\sum_{j=1}^n A_j w_j}\,,\qquad i=1,\dots,n.
\]
Assume that $B_1 \geq B_2 \geq \dots \geq B_n$. Then \(A\) majorizes \(B\) $(A \succ B)$, i.e.
\[
\sum_{i=1}^k A_i \;\ge\; \sum_{i=1}^k B_i \quad \text{for } k=1,\dots,n-1,
\qquad\text{and}\qquad
\sum_{i=1}^n A_i=\sum_{i=1}^n B_i=1.
\]
\end{lemma}

\begin{proof}
Define the following partial sums and weighted average: 
\[
\alpha_k := \sum_{i=1}^k A_i,
\quad
\beta_k := \sum_{i=1}^k B_i
        = \frac{\sum_{i=1}^k A_i w_i}{\sum_{j=1}^n A_j w_j}, 
\quad m_k := \frac{\sum_{i=1}^k A_i w_i}{\sum_{i=1}^k A_i}
     = \frac{\sum_{i=1}^k A_i w_i}{\alpha_k}. 
\]
Since \((w_i)_{i=1}^n\) is increasing, we have \(m_k \in [w_1,w_k]\), hence
\(m_k \le w_k \le w_{k+1}\), and
\[
m_{k+1}
= \frac{\alpha_k m_k + A_{k+1} w_{k+1}}{\alpha_k + A_{k+1}}
\ge \frac{\alpha_k m_k + A_{k+1} m_k}{\alpha_k + A_{k+1}}
= m_k .
\]
Therefore \((m_k)_{k=1}^n\) is nondecreasing in \(k\). In particular,
\(m_k \le m_n\), where \(m_n = \sum_{j=1}^n A_j w_j\). Consequently,
\[
\beta_k
= \frac{\sum_{i=1}^k A_i w_i}{\sum_{j=1}^n A_j w_j}
= \frac{\alpha_k m_k}{m_n}
\le \frac{\alpha_k m_n}{m_n}
= \alpha_k,
\quad k=1,\dots,n.
\]
Also \(\beta_n = \alpha_n = 1\). These are exactly the majorization inequalities:
\[
\sum_{i=1}^k A_i \;\ge\; \sum_{i=1}^k B_i \quad (k=1,\dots,n-1),
\qquad
\sum_{i=1}^n A_i \;=\; \sum_{i=1}^n B_i \;=\; 1.
\]
Hence \(A \succ B\) by Definition \ref{def:majorization}.
\end{proof}

\begin{definition} (Restatement of Strict Schur-Concavity, Definition A.1, Chapter 3 \cite{marshall1979inequalities})
    Let $D \subseteq \mathbb{R}^n$ be a set that is closed under permutations of coordinates (i.e., if $x\in D$ then $Px \in D$ for any permutation matrix $P$).
A function $f:D\to\mathbb{R}$ is called \emph{strictly Schur-concave} if for all $x,y\in D$,
\[
x \succ y \quad \Longrightarrow \quad f(x) \leq f(y),
\]
where $\succeq$ denotes (vector) majorization, and equality holds only when $x$ is a permutation of $y$.
\end{definition}

\begin{lemma}(Example D.1, Chapter 3 \cite{marshall1979inequalities})\label{lem:entropy_is_Schur_concave} Shannon entropy is strictly Schur-concave. 
\end{lemma}

At the end of the proof, we provide a summary of the theoretical results in Table \ref{tab:theory_summary}.

\begin{table*}[!t]
\caption{Summary of Theoretical Results}
\label{tab:theory_summary}
\centering
\begin{tabular}{|l|p{1.2in}|p{0.8in}|p{3.45in}|}
\hline
\textbf{Models} & \multicolumn{2}{c|}{\textbf{Focus}} & \textbf{Theory} \\
\hline
\multirow{4}{*}{\texttt{ATTN}$_M$}
  & \multirow{2}{1.7in}{Training Dynamics} & ODE          & Thm \ref{thm:merged_SAM_ODE}: closed-form solution, convergence boosting term\\ \cline{3-4}
  &                                       & Escape time  & Thm \ref{thm:escape_time}: SAM escapes from saddles slightly faster \\ \cline{2-4}
  & \multirow{2}{1.5in}{Data Upsampling}     & Any Data     & GD becomes similar to SAM if $\|\mSigma\|_F \uparrow$ \\ \cline{3-4}
  &                                       & ``Hard'' Data&  N/A\\ 
\hline
\multirow{4}{*}{\texttt{ATTN}$_S$}
  & \multirow{2}{1.5in}{Training Dynamics}   & ODE          & Thm \ref{thm:SAM_ODE}: SAM has an additional term that changes dynamics\\ \cline{3-4}
  &                                       & Escape time  & Thm \ref{thm:SAM_reduces_SB} SAM reduces simplicity bias \\ \cline{2-4}
  & \multirow{2}{1.5in}{Data Upsampling}     & Any Data     & N/A\\ \cline{3-4}
  &                                       & ``Hard'' Data&  Helps achieve uniform learning speed across features \\
\hline
\end{tabular}
\end{table*}

    \section{Supplemental Experimental Results}
\label{app:aux_results}




Similar to prior findings from \citet{yue2023mammoth}, we observe improvements on SVAMP, DeepMind, and MMLU-Math benchmarks, which contain math tasks different from those in the finetuning data, indicating that targeted upsampling promotes generalization.
More notably, we also observe substantial gains on GSM8K, MATH, and NumGLUE datasets, demonstrating that upsampling not only aids transfer to new tasks formats but also reinforces learning on examples drawn from the original training distribution.
These results establish targeted upsampling as an effective way for mitigating SB and improving generalization. 

\begin{table*}[!ht]
\centering
\caption{
	\textbf{Comparison of upsampling across models and math benchmarks.}
	Each model was finetuned on MathInstruct dataset using AdamW under two settings: without upsampling (FT) and with upsampling (FT+UP), and evaluated in a zero-shot setting using greedy decoding.
	For each model, we report the checkpoint that achieves the highest average accuracy across all datasets, considering the last epoch checkpoints.
	(For comparison, we also report the performance of the pretrained model (PT) without finetuning.)
	The highest performance (within a tolerance of $1.0\%$) is indicated in \textbf{bold}.
	For upsampling, training examples were clustered using loss trajectories from Pythia-70M.
	Datasets include GSM8K, MATH, NumGLUE, SVAMP, DeepMind, and MMLU-Math.
}
\label{tab:adamw-math-upsampling}
\resizebox{1\textwidth}{!}{
\renewcommand{\arraystretch}{1}
\begin{tabular}{lc|ccc|ccc|ccc|ccc}
\toprule
  &  & \multicolumn{3}{c|}{Qwen3-0.6B-Base} & \multicolumn{3}{c|}{Llama3.2-1B} & \multicolumn{3}{c|}{Gemma3-1B-PT} & \multicolumn{3}{c}{Phi2-2.7B} \tabularnewline
\cmidrule(lr){3-5} \cmidrule(lr){6-8} \cmidrule(lr){9-11} \cmidrule(lr){12-14}
Dataset & \#Inst. & PT & FT & FT+UP & PT & FT & FT+UP & PT & FT & FT+UP & PT & FT & FT+UP \tabularnewline
\midrule
GSM8K & $1,319$ & $6.6$ & $\mathbf{59.0}$ & $\mathbf{59.6}$ & $3.0$ & $22.7$ & $\mathbf{25.6}$ & $1.9$ & $8.3$ & $\mathbf{11.1}$ & $56.2$ & $\mathbf{72.3}$ & $\mathbf{72.5}$ \tabularnewline
MATH & $5,000$ & $5.9$ & $36.1$ & $\mathbf{37.9}$ & $3.6$ & $14.6$ & $\mathbf{17.7}$ & $2.3$ & $6.7$ & $\mathbf{9.0}$ & $16.8$ & $33.4$ & $\mathbf{35.0}$ \tabularnewline
NumGLUE & $1,042$ & $9.5$ & $62.5$ & $\mathbf{66.3}$ & $16.4$ & $31.5$ & $\mathbf{36.1}$ & $7.0$ & $24.6$ & $\mathbf{29.8}$ & $34.7$ & $60.7$ & $\mathbf{65.7}$ \tabularnewline
SVAMP & $1,000$ & $6.8$ & $\mathbf{75.6}$ & $\mathbf{75.6}$ & $8.5$ & $\mathbf{41.7}$ & $40.2$ & $3.3$ & $23.8$ & $\mathbf{28.6}$ & $67.8$ & $\mathbf{77.7}$ & $\mathbf{78.5}$ \tabularnewline
DeepMind & $1,000$ & $7.1$ & $\mathbf{65.4}$ & $\mathbf{64.7}$ & $6.0$ & $31.4$ & $\mathbf{34.5}$ & $5.0$ & $20.1$ & $\mathbf{22.8}$ & $33.9$ & $\mathbf{60.2}$ & $58.0$ \tabularnewline
MMLU-Math & $974$ & $24.7$ & $28.2$ & $\mathbf{30.2}$ & $16.2$ & $\mathbf{27.2}$ & $25.8$ & $12.0$ & $20.2$ & $\mathbf{21.6}$ & $12.4$ & $35.0$ & $\mathbf{38.2}$ \tabularnewline
\midrule
Avg. & $1,722$ & $10.1$ & $54.5$ & $\mathbf{55.7}$ & $9.0$ & $28.2$ & $\mathbf{30.0}$ & $5.3$ & $17.3$ & $\mathbf{20.5}$ & $37.0$ & $56.5$ & $\mathbf{58.0}$ \tabularnewline
\bottomrule
\end{tabular}}
\end{table*}

\begin{table*}[!ht]
\caption{\textbf{Upsampling accuracy across optimizers and reference models.} Accuracy (\%) of Phi-2 models trained with AdamW and upsampling-based variants. Reported values show mean accuracy with standard deviation across runs. All upsampling methods use Phi-2 as the base model.}
\label{tab:adamw-upsampling-variants}
\resizebox{1\textwidth}{!}{
\renewcommand{\arraystretch}{1}
\begin{tabular}{l|cccccc}
\toprule
Dataset & Pretrained & AdamW & Upsample Phi Ref & Upsample Pythia Ref & Muon & Upsample Muon \tabularnewline
\midrule
GSM8K & $56.2 \pm 0$ & $72.3 \pm 0$ & $72.6 \pm 1.6$ & $72.7 \pm 1.1$ & $72.6 \pm 0$ & $\mathbf{73.4 \pm 1.1}$ \tabularnewline
MATH & $16.8 \pm 0$ & $33.4 \pm 0$ & $\mathbf{35.1 \pm 0.4}$ & $34.9 \pm 0.1$ & $34.3 \pm 0$ & $\mathbf{35.1 \pm 0.6}$ \tabularnewline
NumGLUE & $34.7 \pm 0$ & $60.7 \pm 0$ & $65.7 \pm 1.3$ & $64.7 \pm 0.2$ & $64.9 \pm 0$ & $\mathbf{66.0 \pm 0.4}$ \tabularnewline
SVAMP & $67.8 \pm 0$ & $77.7 \pm 0$ & $\mathbf{78.6 \pm 0.5}$ & $77.3 \pm 0.9$ & $77.9 \pm 0$ & $77.7 \pm 0.3$ \tabularnewline
DeepMind & $33.9 \pm 0$ & $\mathbf{60.2 \pm 0}$ & $58.1 \pm 0.1$ & $57.9 \pm 0.1$ & $57.5 \pm 0$ & $59.9 \pm 0.1$ \tabularnewline
MMLU-Math & $12.4 \pm 0$ & $35.0 \pm 0.4$ & $\mathbf{38.3 \pm 0.4}$ & $36.6 \pm 0.4$ & $36.1 \pm 0$ & $35.8 \pm 0.4$ \tabularnewline
\midrule
Avg. & $37.0$ & $56.5$ & $\mathbf{58.0}$ & $57.3$ & $57.2$ & $\mathbf{58.0}$ \tabularnewline
\bottomrule
\end{tabular}}
\end{table*}




\begin{table*}[!ht]
\centering
\caption{
	\textbf{Upsampling ablations for Qwen3-0.6B-Base and Llama3.2-1B across math benchmarks.}
	Each model was finetuned under three settings: random upsampling, upsampling \textit{easy} examples, and upsampling hard examples (as described in the paper).
	We ensure that the total number of upsampled examples is identical across ablation experiments; for the easy-upsampling ablation, a random subset of easy examples was selected.
	(For comparison, we also include the performance of the finetuned model (FT) without upsampling.)
	The highest performance \textit{among the ablations} (within a tolerance of $1.0\%$) is indicated in \textbf{bold}; accuracies strictly exceeding the FT baseline are additionally \underline{underlined}.
	All other settings (e.g., clustering model, evaluation setup) are identical to those in Table~\ref{tab:adamw-math-upsampling}.
}
\label{tab:math-upsampling-ablations}
\resizebox{0.64\textwidth}{!}{
\renewcommand{\arraystretch}{1}
\begin{tabular}{l|c|ccc|c|ccc}
\toprule
 & \multicolumn{4}{c|}{Qwen3-0.6B-Base} & \multicolumn{4}{c}{Llama3.2-1B} \tabularnewline
\cmidrule(lr){2-5} \cmidrule(lr){6-9}
Dataset & FT & Rand. & Easy & Hard & FT & Rand. & Easy & Hard \tabularnewline
\midrule
GSM8K & $59.0$ & $56.2$ & $55.0$ & \underline{$\mathbf{59.6}$} & $22.7$ & $22.4$ & $20.4$ & \underline{$\mathbf{25.6}$} \tabularnewline
MATH & $36.1$ & \underline{$36.6$} & \underline{$36.5$} & \underline{$\mathbf{37.9}$} & $14.6$ & \underline{$16.1$} & \underline{$14.9$} & \underline{$\mathbf{17.7}$} \tabularnewline
NumGLUE & $62.5$ & \underline{$64.0$} & $61.4$ & \underline{$\mathbf{66.3}$} & $31.5$ & \underline{$33.5$} & $29.5$ & \underline{$\mathbf{36.1}$} \tabularnewline
SVAMP & $75.6$ & $\mathbf{74.7}$ & $\mathbf{75.3}$ & $\mathbf{75.6}$ & $41.7$ & $\mathbf{40.7}$ & $\mathbf{40.2}$ & $\mathbf{40.2}$ \tabularnewline
DeepMind & $65.4$ & $\mathbf{65.2}$ & $62.4$ & $\mathbf{64.7}$ & $31.4$ & $31.2$ & $30.4$ & \underline{$\mathbf{34.5}$} \tabularnewline
MMLU-Math & $28.2$ & \underline{$28.6$} & $27.1$ & \underline{$\mathbf{30.2}$} & $27.2$ & $\mathbf{25.7}$ & $24.2$ & $\mathbf{25.8}$ \tabularnewline
\midrule
Avg. & $54.5$ & $54.2$ & $53.0$ & \underline{$\mathbf{55.7}$} & $28.2$ & \underline{$28.3$} & $26.6$ & \underline{$\mathbf{30.0}$} \tabularnewline
\bottomrule
\end{tabular}}
\end{table*}




\begin{table*}[!ht]
\centering
\caption{
	\textbf{Upsampling via variational problem synthesis for Qwen3-0.6B-Base, Llama3.2-1B, and Phi2-2.7B across math benchmarks.}
	Depending on our compute constraints, we finetuned each model under the following settings: upsampling hard examples (as described in the paper), upsampling using synthesized examples, and/or a combination of both.
	(For comparison, we also include the performance of the finetuned model (FT) without upsampling.)
	The highest performance \textit{among the upsampling variants} (within a tolerance of $1.0\%$) is indicated in \textbf{bold}; accuracies strictly exceeding the FT baseline are additionally \underline{underlined}.
	All other settings (e.g., clustering model, evaluation setup) are identical to those in Table~\ref{tab:adamw-math-upsampling}.
}
\label{tab:math-upsampling-synthesis}
\resizebox{0.88\textwidth}{!}{
\renewcommand{\arraystretch}{1}
\begin{tabular}{l|c|ccc|c|ccc|c|cc}
\toprule
 & \multicolumn{4}{c|}{Qwen3-0.6B-Base} & \multicolumn{4}{c|}{Llama3.2-1B} & \multicolumn{3}{c}{Phi2-2.7B} \tabularnewline
\cmidrule(lr){2-5} \cmidrule(lr){6-9} \cmidrule(lr){10-12}
Dataset & FT & Upspl. & Synth. & Both & FT & Upspl. & Synth. & Both & FT & Upspl. & Synth. \tabularnewline
\midrule
GSM8K & $59.0$ & \underline{$\mathbf{59.6}$} & $\mathbf{58.6}$ & \underline{$\mathbf{59.4}$} & $22.7$ & \underline{$25.6$} & \underline{$23.3$} & \underline{$\mathbf{27.0}$} & $72.3$ & \underline{$\mathbf{72.5}$} & $\mathbf{72.0}$ \tabularnewline
MATH & $36.1$ & \underline{$\mathbf{37.9}$} & \underline{$\mathbf{37.6}$} & \underline{$\mathbf{37.9}$} & $14.6$ & \underline{$\mathbf{17.7}$} & \underline{$\mathbf{17.0}$} & \underline{$\mathbf{17.8}$} & $33.4$ & \underline{$\mathbf{35.0}$} & \underline{$33.6$} \tabularnewline
NumGLUE & $62.5$ & \underline{$\mathbf{66.3}$} & \underline{$63.1$} & \underline{$\mathbf{65.7}$} & $31.5$ & \underline{$\mathbf{36.1}$} & \underline{$32.8$} & \underline{$34.6$} & $60.7$ & \underline{$\mathbf{65.7}$} & \underline{$61.4$} \tabularnewline
SVAMP & $75.6$ & $\mathbf{75.6}$ & \underline{$\mathbf{75.9}$} & $\mathbf{75.6}$ & $41.7$ & $40.2$ & \underline{$42.6$} & \underline{$\mathbf{46.2}$} & $77.7$ & \underline{$\mathbf{78.5}$} & \underline{$\mathbf{79.5}$} \tabularnewline
DeepMind & $65.4$ & $\mathbf{64.7}$ & $\mathbf{65.0}$ & $\mathbf{64.3}$ & $31.4$ & \underline{$34.5$} & \underline{$35.2$} & \underline{$\mathbf{39.1}$} & $60.2$ & $\mathbf{58.0}$ & $\mathbf{57.7}$ \tabularnewline
MMLU-Math & $28.2$ & \underline{$\mathbf{30.2}$} & $22.6$ & $28.0$ & $27.2$ & $\mathbf{25.8}$ & $23.9$ & $\mathbf{26.5}$ & $35.0$ & \underline{$\mathbf{38.2}$} & $32.9$ \tabularnewline
\midrule
Avg. & $54.5$ & \underline{$\mathbf{55.7}$} & $53.8$ & \underline{$\mathbf{55.2}$} & $28.2$ & \underline{$30.0$} & \underline{$29.1$} & \underline{$\mathbf{31.9}$} & $56.5$ & \underline{$\mathbf{58.0}$} & $56.2$ \tabularnewline
\bottomrule
\end{tabular}}
\end{table*}




\begin{table*}[!ht]
\centering
\caption{
	\textbf{Loss clustering ablations for Qwen3-0.6B-Base, Llama3.2-1B, and Phi2-2.7B across math benchmarks.}
	Each model was finetuned under two settings: upsampling using loss trajectory across multiple proxy model checkpoints (as described in the paper), and upsampling using loss \textit{at} the final proxy model checkpoint.
	(For comparison, we also include the performance of the finetuned model (FT) without upsampling.)
	The highest performance \textit{among the ablations} (within a tolerance of $1.0\%$) is indicated in \textbf{bold}; accuracies strictly exceeding the FT baseline are additionally \underline{underlined}.
	All other settings (e.g., clustering model, evaluation setup) are identical to those in Table~\ref{tab:adamw-math-upsampling}.
}
\label{tab:math-ckpt-vs-traj}
\resizebox{0.7\textwidth}{!}{
\renewcommand{\arraystretch}{1}
\begin{tabular}{l|c|cc|c|cc|c|cc}
\toprule
  & \multicolumn{3}{c|}{Qwen3-0.6B-Base} & \multicolumn{3}{c|}{Llama3.2-1B} & \multicolumn{3}{c}{Phi2-2.7B} \tabularnewline
\cmidrule(lr){2-4} \cmidrule(lr){5-7} \cmidrule(lr){8-10}
Dataset & FT & Traj. & Ckpt. & FT & Traj. & Ckpt. & FT & Traj. & Ckpt. \tabularnewline
\midrule
GSM8K & $59.0$ & \underline{$\mathbf{59.6}$} & \underline{$\mathbf{59.4}$} & $22.7$ & \underline{$\mathbf{25.6}$} & \underline{$\mathbf{25.6}$} & $72.3$ & \underline{$\mathbf{72.5}$} & \underline{$\mathbf{72.6}$} \tabularnewline
MATH & $36.1$ & \underline{$\mathbf{37.9}$} & \underline{$\mathbf{37.1}$} & $14.6$ & \underline{$\mathbf{17.7}$} & \underline{$\mathbf{17.4}$} & $33.4$ & \underline{$\mathbf{35.0}$} & \underline{$\mathbf{35.0}$} \tabularnewline
NumGLUE & $62.5$ & \underline{$\mathbf{66.3}$} & \underline{$\mathbf{65.4}$} & $31.5$ & \underline{$\mathbf{36.1}$} & \underline{$33.5$} & $60.7$ & \underline{$\mathbf{65.7}$} & \underline{$\mathbf{65.5}$} \tabularnewline
SVAMP & $75.6$ & $\mathbf{75.6}$ & \underline{$\mathbf{76.6}$} & $41.7$ & $40.2$ & $\mathbf{41.3}$ & $77.7$ & \underline{$\mathbf{78.5}$} & $75.8$ \tabularnewline
DeepMind & $65.4$ & $\mathbf{64.7}$ & \underline{$\mathbf{65.5}$} & $31.4$ & \underline{$\mathbf{34.5}$} & \underline{$\mathbf{34.9}$} & $60.2$ & $\mathbf{58.0}$ & $56.2$ \tabularnewline
MMLU-Math & $28.2$ & \underline{$30.2$} & \underline{$\mathbf{31.9}$} & $27.2$ & $25.8$ & \underline{$\mathbf{27.8}$} & $35.0$ & \underline{$\mathbf{38.2}$} & \underline{$36.1$} \tabularnewline
\midrule
Avg. & $54.5$ & \underline{$\mathbf{55.7}$} & \underline{$\mathbf{56.0}$} & $28.2$ & \underline{$\mathbf{30.0}$} & \underline{$\mathbf{30.1}$} & $56.5$ & \underline{$\mathbf{58.0}$} & \underline{$56.9$} \tabularnewline
\bottomrule
\end{tabular}}
\end{table*}

    \section{Variational Problem Synthesis}
\label{app:var_synth}

\begin{tcolorbox}[
    breakable,
    title=Prompt used for variational problem synthesis,
]
\begin{verbatim}
You are creating a math problem with this EXACT target solution:

<solution>
{soln}
</solution>

STEP 1 - Create Problem:
Write a problem that should lead to this solution. Match the type:
- Numbers → "Calculate", "Find", "Determine"
- Equations → "Simplify", "Solve for", "Express"  
- Code → Create a realistic scenario requiring a program (do not say 
  "write code for this solution")
- Proofs → "Prove that", "Show that"

Include all necessary information. Use high school/college level context.

STEP 2 - Verify Your Problem:
Solve the problem you created step-by-step. Show your work.

STEP 3 - Check Match:
Does your solution match the target? Compare carefully:
- Same numbers/values?
- Same format?
- Same type?

If NO MATCH → Revise your problem in Step 1 and repeat Steps 2-3.
If MATCH → Proceed to output.

FINAL OUTPUT (only after verification):

<question>
[Your verified problem - do NOT include the solution]
</question>

<verification>
[Brief confirmation: "Verified - solving gives {soln}"]
</verification>
\end{verbatim}
\end{tcolorbox}

\end{appendices}


\end{document}